\documentclass{ecai}

\usepackage{times}
\usepackage{graphics}

\usepackage{balance}
\usepackage{amsmath}
\usepackage{amssymb}
\usepackage{amsthm}
\usepackage{amsfonts}
\usepackage{booktabs}
\usepackage{siunitx}
\usepackage{bm}
\usepackage{multirow}
\usepackage{makecell}

\usepackage{mathtools}
\usepackage{bbm}

\usepackage{algorithm}
\usepackage{algpseudocode}

\newtheorem{lemma}{Lemma}

\newtheorem{definition}{Definition}

\newtheorem{thm}{Theorem}

\begin{document}

\title{Randomized Kernel Multi-view Discriminant Analysis}

\author{Xiaoyun Li, \  Jie Gui, \ and Ping Li\\
{\normalfont Department of Statistics, Rutgers University}\\
{\normalfont Cognitive Computing Lab, Baidu Research}\\
{\normalfont xiaoyun.li@rutgers.edu, \{guijiejie, pingli98\}gmail.com}
}

\maketitle
\bibliographystyle{ecai}

\begin{abstract}
	In many artificial intelligence and computer vision systems, the same object can be observed at distinct viewpoints or by diverse sensors, which raises the challenges for recognizing objects from different, even heterogeneous views. Multi-view discriminant analysis (MvDA) is an effective multi-view subspace learning method, which finds a discriminant common subspace by jointly learning multiple view-specific linear projections for object recognition from multiple views, in a non-pairwise way. In this paper, we propose the kernel version of multi-view discriminant analysis, called kernel multi-view discriminant analysis (KMvDA). To overcome the well-known computational bottleneck of kernel methods, we also study the performance of using random Fourier features (RFF) to approximate Gaussian kernels in KMvDA, for large scale learning. Theoretical analysis on stability of this approximation is developed. We also conduct experiments on several popular multi-view datasets to illustrate the effectiveness of our proposed strategy.
\end{abstract}

\section{Introduction}

Multi-view learning~\cite{Article:Kan_PAMI16,Proc:Trapp_UAI17,Proc:Xin_UAI17,Proc:Gui_ICDM18} or learning with multiple different feature sets is rapidly growing research area with practical success in  important applications. For example, a person can be described by visual light face image, sketch, near infrared face image, iris, fingerprint, palmprint or signature with information secured from many different sources (e.g., distinct angles). Our task is to classify an object from one view, given the information from other views. For instance, Figure~\ref{fig:CUFSF-example} shows some samples from two different views in the CUFSF multi-view dataset (more detailed dataset description is provided in the experiments section). Here, each person represents a object class, and we would like to classify a sketch given all the label information of the photos, or vise versa. In many cases, the views can be quite different as well. An example is the content-based web-image retrieval, where an object can be identified by the text depicting the image or visual features from the image itself. Here the text and image can be regarded as two distinct views, of the~same~object.

In this paper, we focus on multi-view subspace learning, which aims to learn a common subspace shared by all different views. The research on multi-view learning started with ``two-view'' learning. The canonical correlation analysis (CCA)~\cite{Proc:Hotelling_1936,Book:HTF09,Proc:Xu_NeurIPS19} is perhaps the most well-known two-view unsupervised algorithm. CCA finds the linear projections for two views respectively which have maximum correlation with each other. The discriminative variants of CCA were studied in~\cite{Proc:Ma_ICML07,Proc:Sun_ICDM08}. The paper~\cite{Proc:Lin_ECCV06} provided common discriminant feature extraction  to maximize the inter-class separability and meanwhile minimize the intra-class scatter. Multi-view CCA (MCCA)~\cite{Article:Nielsen_TIP02,Proc_Rupnik10} was proposed to secure one common subspace for all views, under unsupervised setting. Generalized multi-view analysis (GMA) framework~\cite{Proc:Sharma_CVPR12} took advantage of class information, resulting in a discriminant common space.  The authors of~\cite{Proc:Sim_ICCV09} presented the multi-model discriminant analysis (MMDA) to decompose variations in a dataset into independent modes (factors). The multi-view discriminant analysis (MvDA) was proposed by~\cite{Article:Kan_PAMI16}, which learned projections for different views jointly via Fisher discriminant analysis. In~\cite{Article:Cao_2017}, the authors proposed a  variant called MvMDA, which differs from standard MvDA in the definition of inter-class and intra-class covariance matrices.

\begin{figure}
	\begin{center}
		\includegraphics[width=3.47in]{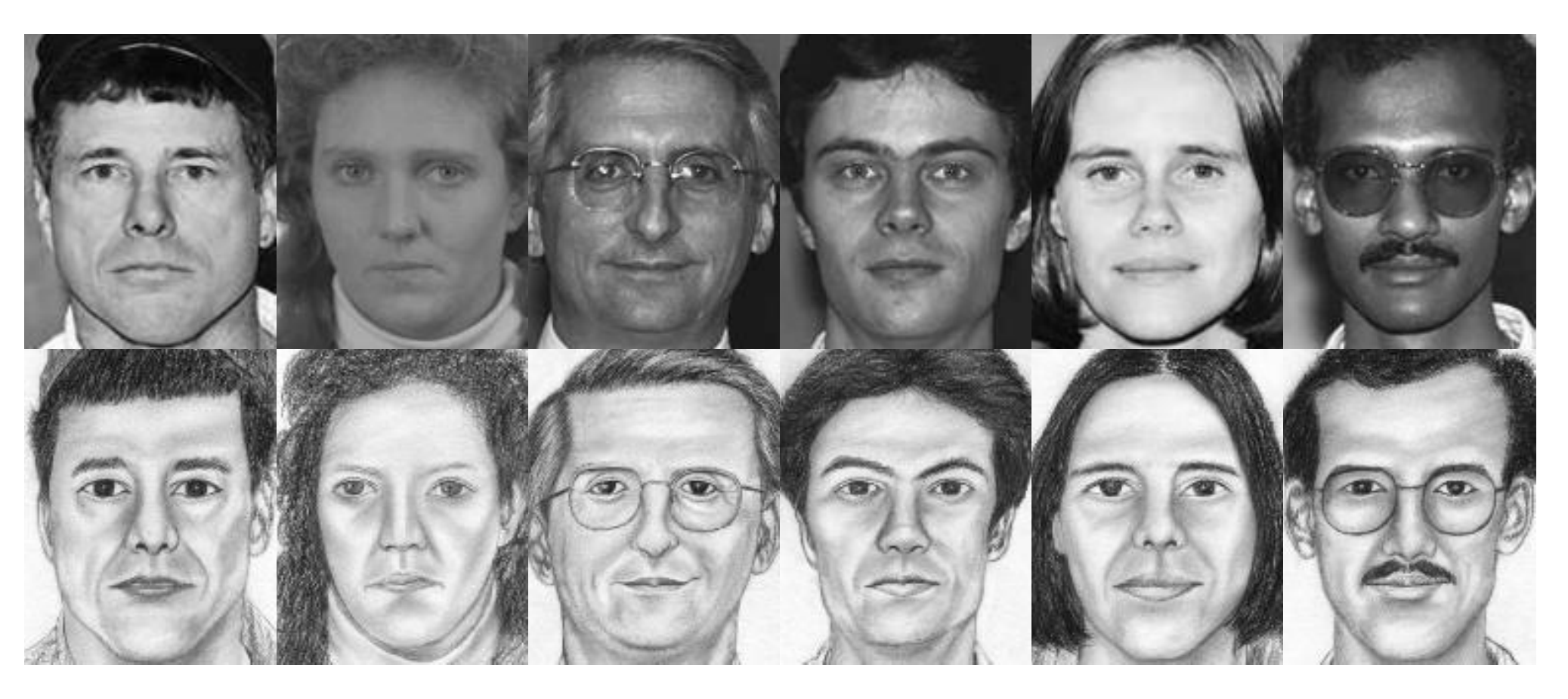}
	\end{center}
 	\vspace{-0.2in}
	\caption{Examples from CUFSF multi-view dataset. View 1 (first row): actual face photo. View 2 (second row): sketch drawn by artist.}
	\label{fig:CUFSF-example}
\end{figure}

In many cases, non-linear kernel has stronger learning capacity than linear kernel. Hence, to enhance performance of standard MvDA algorithm, in this paper we seek to kernelize multi-view discriminant analysis, and derive so-called kernel MvDA (KMvDA). However, it is known that a direct implementation of nonlinear kernels is difficult for large-scale datasets, since even for a medium-sized dataset with only 100,000 instances, the 100,000$\times$100,000 kernel matrix has $10^{10}$ entries, which is essentially not feasible for most machines that people use daily. Therefore, in practical applications, being capable of linearizing nonlinear kernels is highly welcome~\cite{Book:Bottou_07,Proc:HashLearning_NIPS11,Proc:Li_NIPS13}. Random Fourier features (RFF's)~\cite{Proc:Rahimi_NIPS07,Proc:Li_KDD17} is a celebrated algorithm to linearly approximate Gaussian (RBF) kernel, which has been widely used and studied in literature ~\cite{Proc:Lopez_ICML14,Proc:Sutherland_UAI15} on clustering, CCA, PCA, classification, etc. For two data points, the inner product of linearized Fourier features is unbiased estimator of the true RBF kernel. By using faster linearized algorithms, we are able to exploit the learning power of non-linear kernel (e.g., RBF kernel) in linear time when dealing with large-scale datasets.

\vspace{0.05in}

\noindent \textbf{Our contributions.}\ \  In this paper, we first derive a kernelized MvDA, and then apply random Fourier features to KMvDA and demonstrate its feasibility for large scale learning. To the best of our knowledge, this is the first attempt in literature to randomize multi-view discriminant learning. It is shown that by approximation, the change in eigenspace (and hence the projections) could be bounded and converges to zero as we increase the number of random features. Experimental results provide evidence on the advantage of KMvDA, as well as the effectiveness of the linearized approximation.

\vspace{0.05in}
\noindent \textbf{Roadmap.}\ \  In Section 2, we introduce some preliminaries on multi-view discriminant analysis (MvDA) and eigenspace comparison. In Section 3, we formulate the kernel MvDA (KMvDA). In Section 4, we introduce the kernel approximation scheme and provide theoretical analysis of the approximation error on the subspace learned by KMvDA. In Section 5, we conduct experiments to show the effectiveness of our method. In the last section, we discuss some relevant topics and conclude the paper.

\section{Preliminaries}

\subsection{Problem setting and notations}

In this paper, we denote a  multi-view dataset by $X = \left\{ {{x_{ijk}}\left| {i = 1, \cdots ,c; j = 1, \cdots ,v; k = 1, \cdots ,{n_{ij}}} \right.} \right\}$ with the instances where ${x_{ijk}} \in {R^{{d_j}}}$ is the $k$-th instance from the $i$-th class of the $j$-th view of $d_j$ dimension, $c$ denotes the number of classes, $v$ is the number of views. $X_j$ represents the instances from the $j$-th view. $n_{ij}$ denotes the number of instances from the $i$-th class of the $j$-th view, and $n_i$ is the number of observations from the $i$-th class of all views. Let $n$ denote the total number of examples from all views. Let $\mathcal C(x)$ denote the class label of $x$. Let $w_1,...,w_v$ denote the view-specific linear projections that we aim to learn. Throughout the paper, $\Vert\cdot\Vert_F$ denotes matrix Frobenius norm and $\Vert\cdot\Vert$ is the operator norm for matrix and Euclidean norm for vector.

\begin{figure}
	\begin{center}
		\includegraphics[width=3.3in]{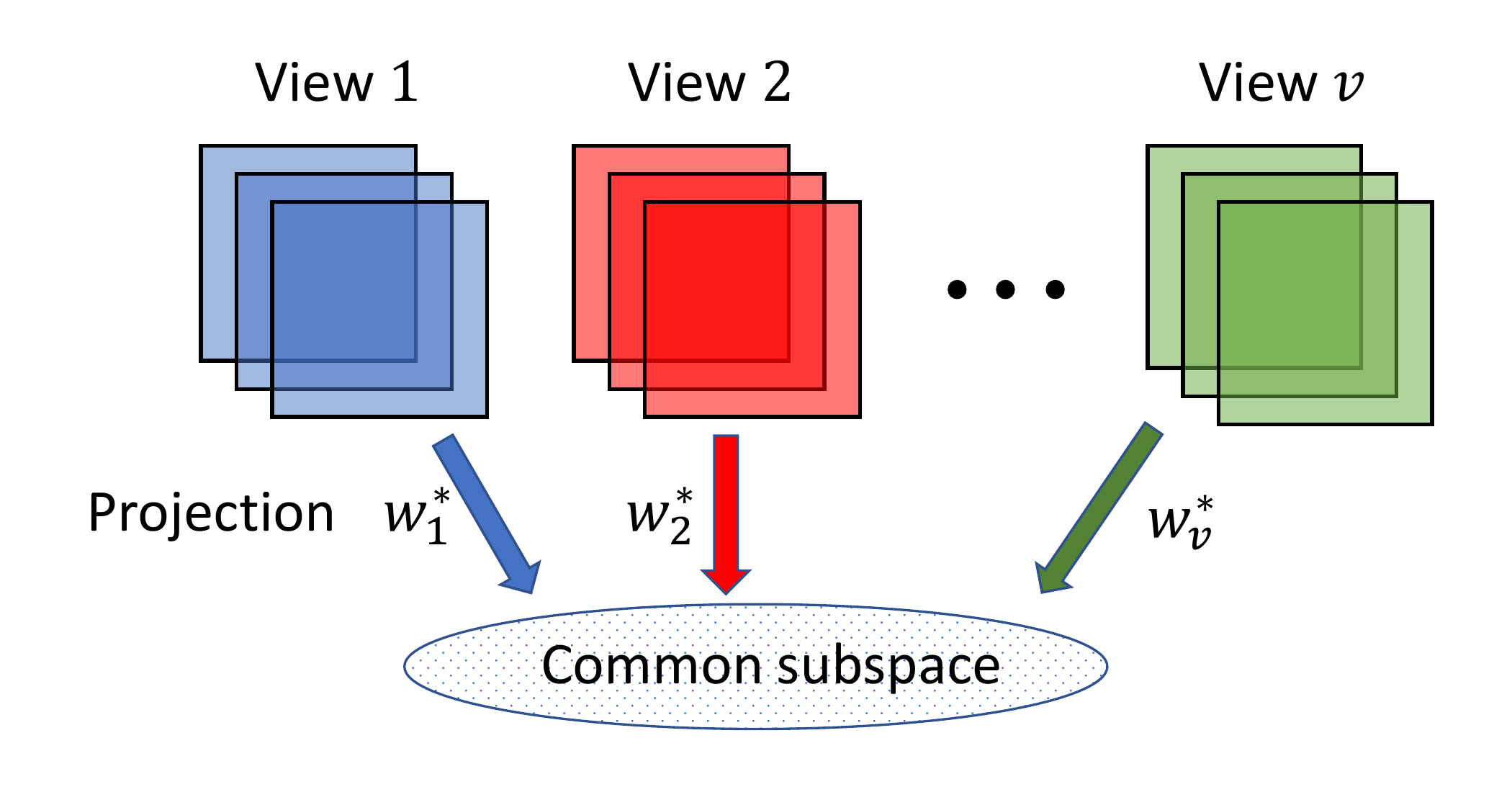}
	\end{center}
 	\vspace{-0.25in}
	\caption{
	An illustration of the MvDA framework, which jointly learns a projection for each view and conducts classification in the common subspace.
	}	\label{fig:MvDA-framework}
\end{figure}

\vspace{-0.2in}

\subsection{Multi-view discriminant analysis}

Multi-view discriminant analysis (MvDA)~\cite{Article:Kan_PAMI16} aims to find $v$ view-specific linear projections $w_1, w_2, \cdots , w_v$ which can respectively transform the instances from $v$ views to one discriminant common space, by minimizing the within-class variation and maximizing the between-class variation. Instances from $v$ views are then projected onto the same common space by $w_1, w_2, \cdots , w_v$. Figure~\ref{fig:MvDA-framework} depicts the idea and framework of MvDA. To achieve cross-view discrimination, the within-class variation from all views should be minimized while the between-class variation from all views should be maximized in the common space.

More specifically, MvDA is a generalization of linear discriminant analysis (LDA)~\cite{Book:HTF09} for multi-view learning. They share the same type of objective function (i.e., the Rayleigh quotient),
\begin{eqnarray}\label{eqn:1}
\left( {w_1^*,w_2^*, \cdots ,w_v^*} \right) = \arg \mathop {\max }\limits_{{w_1}, \cdots ,{w_v}} tr\left( {\frac{{{W^T}DW}}{{{W^T}SW}}} \right),
\end{eqnarray}
where $S=\left[ {\begin{array}{*{20}{c}}
	{{S_{11}}} &  \cdots  & {{S_{1v}}}  \\
	\vdots  &  \vdots  &  \vdots   \\
	{{S_{v1}}} &  \cdots  & {{S_{vv}}}  \\
	\end{array}} \right]$, $D=\left[ {\begin{array}{*{20}{c}}
	{{D_{11}}} &  \cdots  & {{D_{1v}}}  \\
	\vdots  &  \vdots  &  \vdots   \\
	{{D_{v1}}} &  \cdots  & {{D_{vv}}}  \\
	\end{array}} \right]$.
	
\noindent The terms $S$ and $D$ can be seen as the within-class scatter matrix and between-class scatter matrix for multi-view learning, respectively. The $r$-th column and the $j$-th row block matrix of $S$, which is denoted by $S_{jr}$, is defined as:
\begin{eqnarray}\label{eqn:3}
{S_{jr}} = \left\{ \begin{array}{l}
\sum\nolimits_{i = 1}^c {\left( {\sum\nolimits_{k = 1}^{{n_{ij}}} {{x_{ijk}}x_{ijk}^T - \frac{{{n_{ij}}{n_{ij}}}}{{{n_i}}}u_{ij}^{\left( x \right)}u_{ij}^{{{\left( x \right)}^T}}} } \right)} ,\ \ \ j = r, \\
- \sum\nolimits_{i = 1}^c {\frac{{{n_{ij}}{n_{ir}}}}{{{n_i}}}u_{ij}^{\left( x \right)}u_{ir}^{{{\left( x \right)}^T}},} \hspace{0.9in} otherwise. \\
\end{array} \right. \nonumber
\end{eqnarray}

\noindent The term $D_{jr}$ is the $r$-th column and the $j$-th row block matrix of $D$ and defined as

{\fontsize{8}{8}\selectfont
\begin{equation*}\label{eqn:2}
{D_{jr}} = \left( {\sum\limits_{i = 1}^c {\frac{{{n_{ij}}{n_{ir}}}}{{{n_i}}}u_{ij}^{\left( x \right)}u_{ir}^{{{\left( x \right)}^T}}} } \right)- \frac{1}{n}\left( {\sum\limits_{i = 1}^c {{n_{ij}}u_{ij}^{\left( x \right)}} } \right){\left( {\sum\limits_{i = 1}^c {{n_{ir}}u_{ir}^{\left( x \right)}} } \right)^T},
\end{equation*}}\\
\noindent with $u_{ij}^{\left( x \right)} = \frac{1}{{{n_{ij}}}}\sum\nolimits_{k = 1}^{{n_{ij}}} {{x_{ijk}}}$.

\vspace{0.1in}

Basically, MvDA extends LDA to multi-view setting with carefully designed block covariance matrices that aim to achieve accurate multi-view classification. The standard approach for solving the optimization problem (\ref{eqn:1}) is by transforming it into a generalized eigenvalue problem, which will be introduced in the next sub-section.

\subsection{Eigenspace Comparison}

A standard eigenvalue problem (SEP), given a square matrix A, is to solve $Aw=\lambda w$ for vector $w$ and scalar $\lambda$. Feasible $\lambda$'s are called the eigenvalues (or spectrums), and $w$'s are called eigenvectors. For SEP, there exist many well-known results on the eigenvalues (e.g., Weyl's theorem) when a small perturbation is added to $A$. The  Davis-Kahan theorem (e.g., the $\sin\Theta$ theorem) provides bounds on the change of angles between eigenvectors, which could be regarded as a measure of the change of eigenspace. These theorems cast additional restrictions on the eigenvalues by assuming the existence of eigengaps.

A generalized eigenvalue problem (GEP), given $A,B\in\mathbbm R^{n\times n}$, is to find the solution to the system
\begin{equation}
    \beta Ax=\alpha Bx,  \label{GEP}
\end{equation}
and each pair of $(\alpha,\beta)$ and $x$ that satisfies this equation is called a pair of generalized eigenvalue and generalized eigenvector. One natural idea to study the perturbation of generalized eigen system (\ref{GEP}) is to left-multiply the inverse of $B$ on both sides and yields an ordinary eigenvalue problem $B^{-1}Ax=\frac{\alpha}{\beta} x$, provided that $B$ is invertible. However, when $B$ is singular, this approach would fail but feasible solution to GEP may still exists~\cite{Article:Stewart_1979,Article:Sun_1983,Article:Crawford_1976}. More specifically, when matrices $A$ and $B$ have common null space, the set of eigenvalues may become the whole complex plane. In this case, the problem is said to be \textit{ill-disposed} since the spectrum is extremely unstable. The \textit{Crawford number}, defined as
\begin{equation}
    \mathcal C(A,B)=\min_{\Vert x\Vert=1}\{|x^H(A+iB)x|\}, \label{crawford}
\end{equation}
is very important in this context. Here $x^H$ means the conjugate transpose. For real matrices, we could also write it as $\mathcal C(A,B)=\displaystyle\min_{\Vert x\Vert=1}\{(x^TAx)^2+(x^TBx)^2\}^{1/2}$. Matrix pair $(A,B)$ is said to be definite if $\mathcal C(A,B)>0$ holds. In this case, the problem is called a definite problem. This technical condition ensures that $A$ and $B$ not having interlacing null space.

It is shown in~\cite{Article:Stewart_1979} that without special information, the eigenvectors may be very sensitive to small perturbations, but the subspace spanned by them may be stable. In the following, we summarize some related concepts and results on subspace perturbation. Notations with tildes denote the counterparts in perturbed problem.

\begin{definition}
	Suppose $(A,B)$ is a definite matrix pair. A subspace $\mathcal X$ is an eigenspace of $(A,B)$ if $\dim(A\mathcal X+B\mathcal X)\leq\dim\mathcal X.$
\end{definition}
For a definite pair, there always exists $Z=(Z_1,Z_2)$ with $Z_1\in C_{n\times l}$, $Z_2\in C_{n\times(n-l)}$, such that
\begin{equation}
Z^HAZ=\begin{pmatrix}
A_1 & 0\\0 & A_2
\end{pmatrix},\quad
Z^HBZ=\begin{pmatrix}
B_1 & 0\\0 & B_2,
\end{pmatrix}  \label{decomp}
\end{equation}
where $A_1,B_1\in C^{l\times l}$, $Z_1^HZ_1=I_l$ and $Z_2^HZ_2=I_{n-l}$, and a similar decomposition holds for perturbed matrices. Clearly, $Z_1$ is an eigenspace for $(A,B)$. Let $\mathcal R(A)$ be the column space of $A$. Analyzing a rotation between $\mathcal R(Z_1)$ and $\mathcal R(\tilde Z_1)$ shows that
\begin{equation}
\Theta=\cos^{-1}(Z_1^H\tilde{Z_1}\tilde{Z_1}^HZ_1)^{1/2} \label{angle}
\end{equation}
represents the canonical angles between some sets of suitably chosen base vectors of $\mathcal R(Z_1)$ and $\mathcal R(\tilde Z_1)$. Hence, $\sin\Theta$ becomes a good measure of the difference between these two subspaces. The following theorem depicts the relationship between $\sin\Theta$, the gap between subspaces and corresponding projection operators.

\begin{thm}\cite{Article:Stewart_1979}
	Let $P_{\mathcal R}$ and $P_{\mathcal{\tilde R}}$ be the orthogonal projections onto $\mathcal R(Z_1)$ and $\mathcal R(\tilde Z_1)$. Let $\Theta$ be defined by (\ref{angle}). Furthermore, define the gap between subspaces $\mathcal R\triangleq\mathcal R(Z_1)$ and $\tilde{\mathcal R}\triangleq\mathcal R(\tilde Z_1)$ as $\mathcal G(\mathcal R, \tilde{\mathcal R})=\max\{\displaystyle\sup_{\substack{\Vert x\Vert=1 \\ x\in \mathcal R}}\inf_{y\in\tilde{\mathcal R}}\Vert x-y\Vert,\displaystyle\sup_{\substack{\Vert y\Vert=1 \\ y\in \tilde{\mathcal R}}}\inf_{x\in\mathcal R}\Vert x-y\Vert\}$, then
	\begin{align*}
	    &\mathcal G(\mathcal R, \tilde{\mathcal R})=\Vert P_{\mathcal R}-P_{\mathcal{\tilde R}}\Vert=\Vert\sin\Theta\Vert,\\
	    &\sqrt 2\mathcal G(\mathcal R, \tilde{\mathcal R})\leq\Vert P_{\mathcal R}-P_{\mathcal{\tilde R}}\Vert_F=\sqrt 2\Vert\sin\Theta\Vert_F.
	\end{align*}\label{dist}
\end{thm}

This equivalence makes $\sin\Theta$ a commonly used measure for the difference between two subspaces. We also define the \textit{chordal distance} between  points $\bm{p_1}=(a_1,b_1)$, $\bm{p_2}=(a_2,b_2)$~as
\begin{align}
&\rho(\bm{p_1},\bm{p_2})
=\frac{|a_1b_2-a_2b_1|}{\sqrt{|a_1|^2+|b_1|^2}\sqrt{|a_2|^2+|b_2|^2}}, \label{chordal}
\end{align}
which is crucial for comparing eigenvalues in generalized eigen problems. It is invariant under rotation about the origin and can handle large, or infinite eigenvalues by measuring the distances on the Riemann sphere.

\section{Kernel Multi-view Discriminant Analysis}

For many linear learners, kernel trick enables us to access a much higher, possibly infinite dimensional feature space by operating in an inner product space associated with a proper Reproducing Kernel Hilbert Space (RKHS)~\cite{Article:Aronszajn_1950}. Examples of kernel methods include kernel PCA, kernel SVM, etc. In this section, we combine kernel trick with MvDA and derive kernel MvDA (KMvDA).

\subsection{KMvDA}

\textbf{Formulation.}\hspace{0.1in}Without loss of generality, we look at one projection direction (e.g., the top eigenvector). Based on previous definitions, we rewrite $S_{jr}$ and ${D_{jr}}$ in matrix form:
\begin{align}\label{H_{jr}^S}
{S_{jr}} \triangleq& X_jH_{jr}^SX_r \nonumber\\
=&
\begin{cases}
{X_j}\left( {I - \sum\limits_{i = 1}^c {\frac{1}{{{n_i}}}e_j^i{{\left( {e_j^i} \right)}^T}} } \right)X_r^T , \quad j = r,\\
{X_j}\left( { - \sum\limits_{i = 1}^c {\frac{1}{{{n_i}}}e_j^i{{\left( {e_r^i} \right)}^T}} } \right)X_r^T, \quad  otherwise,
\end{cases}
\end{align}
\begin{align}\label{H_{jr}^D}
{D_{jr}}=& \sum\limits_{i = 1}^c {\frac{1}{{{n_i}}}} \left( {{n_{ij}}\mu _{ij}^{\left( x \right)}} \right){\left( {{n_{ir}}\mu _{ir}^{\left( x \right)}} \right)^T} \nonumber\\
&\qquad\qquad  - \frac{1}{n}\left( {\sum\limits_{i = 1}^c {\sum\limits_{k = 1}^{{n_{ij}}} {{x_{ijk}}} } } \right){\left( {\sum\limits_{i = 1}^c {\sum\limits_{k = 1}^{{n_{ir}}} {{x_{irk}}} } } \right)^T} \nonumber\\
=& \sum\limits_{i = 1}^c {\frac{1}{{{n_i}}}} {X_j}e_j^i{\left( {{X_r}e_r^i} \right)^T} - \frac{1}{n}{X_j}{e_j}{\left( {{X_r}{e_r}} \right)^T} \nonumber\\\notag
=& {X_j}\left( {\sum\limits_{i = 1}^c {\frac{1}{{{n_i}}}e_j^i{{\left( {e_r^i} \right)}^T} - \frac{1}{n}{e_j}e_r^T} } \right)X_r^T\\
\triangleq& X_jH_{jr}^DX_r,
\end{align}

\noindent where ${e_r}$ is a vector with all elements equal to one and the dimensionality of ${e_r}$ is the same as the number of the examples of the $r$-th view; $e_r^i$ is a vector whose dimensionality is the same as that of ${e_r}$ and with the $i$-th class equal to one and zero otherwise. In the rest of this section, the same computation is described in another inner product space $\mathcal F$, which is associated with the input space by map $\phi :{\mathbbm R^d} \to \mathcal F,\;x \mapsto \phi \left( x \right)$ and a kernel function $k:\mathcal X\times\mathcal X\rightarrow \mathbbm R$ in a reproducing kernel Hilbert space (RKHS) such that for $\forall x,y\in \mathcal X$,
$$k(x,y)=\langle\phi(x),\phi(y)\rangle_{\mathcal F}.$$

Note that the feature space $\mathcal F$ could have an arbitrarily large, possibly infinite dimensionality. However, explicit representation of the function $\phi(\cdot)$ is unnecessary as long as $\mathcal F$ is a proper inner product space. By this mapping, the objective function of KMvDA becomes
\begin{align*}
J=\frac{\displaystyle\sum_{j = 1}^v \sum_{r = 1}^v w_j^T\phi(X_j)( \sum_{i = 1}^c \frac{1}{n_i}e_j^i(e_r^i)^T - \frac{1}{n}{e_j}e_r^T)\phi(X_r^T)w_r }{\splitfrac{\bigg(\displaystyle\sum_{j = 1}^v w_j^T \phi(X_j)( I - \sum_{i = 1}^c \frac{1}{n_i}e_j^i(e_j^i)^T)\phi(X_j^T)w_j}{+ \displaystyle\sum_{j = 1}^v {\sum_{r = 1,r \ne j}^v w_j^T\phi(X_j)(- \sum_{i = 1}^c \frac{1}{n_i}e_j^i(e_r^i)}^T)\phi(X_r^T)w_r\bigg)} } ,
\end{align*}
where $(w_1,...,w_v)$ are projection directions of distinct views. By the well-known Representer Theorem in RKHS, there exists $z_j$ such that ${w_j} = \phi \left( {{X_j}} \right){z_j},\forall j=1,...,v.$ Therefore, we can re-write the objective function using the inner products in the feature space $\mathcal F$,
\begin{align}\label{H_D,H^S}
J&=\frac{\displaystyle\sum_{j = 1}^v \sum_{r = 1}^v z_j^TK_j(\sum_{i = 1}^c \frac{1}{n_i}e_ie_i^T - \frac{1}{n}ee^T)K_rz_r}{\splitfrac{
	\bigg(\displaystyle\sum_{j = 1}^v z_j^T K_j( I - \sum_{i = 1}^c \frac{1}{n_i}e_ie_i^T)K_jz_j}
	{+ \displaystyle\sum_{j = 1}^v\sum_{r = 1,r \neq j}^v z_j^TK_j(-\sum_{i = 1}^c \frac{1}{n_i}e_ie_i^T)K_rz_r\bigg)}}  \nonumber\\
&\triangleq\frac{z^TK^TH^DKz}{z^TK^TH^SKz}\triangleq \frac{z^TDz}{z^TSz}.
\end{align}
where $H^D$ and $H^S$ are block matrices with entries $H_{jr}^D$, $H_{jr}^S$ respectively, and $K=diag(K_1,\dots,K_v)$ is a block diagonal matrix. After some standard derivation, (\ref{H_D,H^S}) eventually turns into solving the GEP

\begin{equation}
	Dz=\lambda S z. \label{GEP3}
\end{equation}

As the eigenvalues are invariant of scale, we denote in this paper that all eigenvalues for MvDA (and KMvDA) are of the form $(\lambda_i,1)$, along with the paired eigenvectors $z_i$, $i=1,...,n$. Note that, every $z_i$ is an $n$-dimensional vector (recall that $n$ is the total number of samples). The $i$-th projection direction of the $j$-th view, $z_i^j$, is set to be the slice at corresponding positions of the $j$-th view. For our task, we choose  projection directions as the eigenvectors associated with the largest eigenvalues. More precisely, we sort $\lambda_1\geq\lambda_2\geq\dots\geq\lambda_n$ and project $X_j$ onto an $l$-dimensional space with respect to $Z^j=(z_1^j,...,z_l^j)$.

\vspace{0.1in}
\noindent\textbf{Testing phase.}\hspace{0.1in}Given a new test set $Y=(Y_1,...,Y_v)$, for a test example $y=(y_1,...,y_v)$, we compute projections of the $j$-th view in the kernel space by
\begin{align*}\label{eqn:12}
    Proj(y_j)=(W^j)^T\phi \left( {{y_j}} \right)&= (Z^j)^T\phi {\left( {{X_j}} \right)^T}\phi \left( {{y_j}} \right)\\
    &=(Z^j)^T k(X_j,y_j),
\end{align*}
where $y_j$ is the $j$-th view of $y$ and $k(\cdot,y)$ represents the element-wise kernel function.  If our goal is to classify $y_j$ based on view $Y_m$, we assign $y_j$ with the label of nearest neighbor of $Y_m$ in the projected space, $\hat{\mathcal{C}}(y_j)=\mathcal{C}(\displaystyle{\arg\min_{y'\in Y_m}} \|Proj(y')-Proj(y_j)\|)$.

\subsection{Kernels}

In this paper, we focus on comparing two kernels. The linear kernel is simply defined as the inner product between two data points, which will serve as the baseline. For non-linear kernels, we consider the radial basis functions (RBF) kernel (i.e., the Gaussian kernel), which is the most commonly used kernel in statistical learning and many related fields~\cite{Book:HTF09}. The RBF kernel between two examples $x$ and $y$ is computed as
\begin{eqnarray*}\label{eqn:13}
k(x,y)=\exp\left(-\frac{\parallel x-y\parallel^2}{2\sigma^2}\right),
\end{eqnarray*}
where $\sigma^2$ is the kernel width hyper-parameter. It is well-known that the RBF kernel is shift-invariant and positive definite.

\section{KMvDA with Randomized Kernels}
As discussed precedingly, in many practical tasks, computing kernels is very expensive when the data size is large. Therefore, linearized kernels are important in many cases, as one can enjoy the benefits of kernel methods with a linear learner. In this section, we consider linearizing the RBF kernel in the KMvDA approach, which aims at approximating the learning performance of using exact RBF kernel, but in linear time complexity. The tool we use is the random Fourier features (RFF's)~\cite{Proc:Rahimi_NIPS07,Proc:Li_KDD17}.

\subsection{Random Fourier Features (RFF)}
Given a shift-invariant kernel $k(x-y)$, let $p(w)$ be its Fourier transformation. Since the measure $p$ and kernel $k$ are both real, we have
\begin{align*}
k(x,y)=\int e^{jw^T(x-y)}p(w)dw&\overset{\mathrm{Bochner}}{=}E_{p(w)}[e^{jw^T(x-y)}]\\
&\hspace{0.12in}=E_{p(w)}[\cos w^T(x-y)].
\end{align*}
Here, Bochner's theorem reveals that $p(w)$ is a valid non-negative measure if the kernel is continuous positive definite, and hence we can express the kernel as an expectation. Therefore, one can use Monte-Carlo method to estimate the kernel by repeatedly sampling from $p(w)$. The features generated in this way are called random Fourier features (RFF's). For the RBF kernel, based on trigonometric identities, one popular scheme is
\begin{align*}
\dot f_{w,b}(x)=\sqrt{2}\cos(w^Tx+b),
\end{align*}
where $w\sim N(0,1/\sigma^2)$ and $b\sim uniform(0,2\pi)$. This construction achieves unbiasedness, i.e., $E[\dot f_{w,b}(x)^T\dot f_{w,b}(y)]=k(x,y)$. Let $F_i=[\dot f_{w_i,b_i}(x_1) \dots \dot f_{w_i,b_i}(x_n)]^T$, we estimate the RBF kernel matrix by the mean of $i.i.d.$ samples
\begin{equation}
\hat K=\frac{1}{m}\sum_{i=1}^mF_iF_i^T,  \label{K_hat}
\end{equation}
and define estimates of matrices $D$ and $S$ as $\hat D$ and $\hat S$ using $\hat K$ accordingly. Then we solve the problem $\hat Dw=\lambda(\hat S+\epsilon I)w$ to approximate the solution using exact kernel matrices.

Note that it has been shown in~\cite{Proc:Li_KDD17} that one can substantially improve the performance of RFF $\dot f$ by normalizing the random features. \\

Obviously, RFF is tightly related to the method of random projections, which  has become a  popular technique to reduce data dimensionality while preserving distances between data points, as guaranteed by the celebrated Johnson-Lindenstrauss (J-L) Lemma and variants~\cite{Article:JL84,Article:Dasgupta_JL03}. There is a  rich literature of research on the theory and applications of random projections, such as  clustering, classification, near neighbor search, bio-informatics, compressed sensing, quantization,  etc.~\cite{Proc:Indyk_Motwani_STOC98,Proc:Dasgupta_UAI00,Proc:Bingham_KDD01,Article:Buhler_JCB02,Proc:Charikar_STOC02,Proc:Fagin_SIGMOD03,Proc:Fern_ICML03,Book:Vempala04,Proc:Li_Hastie_Church_COLT06,Article:Donoho_CS_JIT06,Article:Candes_Robust_JIT06,Proc:Frund_NIPS07,Proc:Dasgupta_STOC08,Proc:Wang_SDM10,Proc:Dahl_ICASSP13,Proc:Li_NeurIPS19,Proc:Li_NeurIPS19_Asymmetric}.

\subsection{Analysis of Randomized KMvDA}

In this section, we investigate the subspace perturbation of using linearized RFF kernels, which directly determines the approximation efficiency of randomized KMvDA. In this sequel, notations with hats are defined for objects using approximated kernels. Without loss of generality, we assume that the number of examples in each view is the same, i.e., $\tilde{n}=n/v$. Moreover, in each view, the number of observations per class is also the same (all equal to $\tilde{n}/c$). Besides, the classes are ordered in the same way in all views.

\begin{lemma}  \label{norm lemma}
Let $H^S$, $H^D$,$H_{(\cdot,\cdot)}^S$ and $H_{(\cdot,\cdot)}^D$ be defined in (\ref{H_{jr}^S}), (\ref{H_{jr}^D}) and (\ref{H_D,H^S}). For $\forall\ j,r\leq v$, $\Vert H_{jr}^D \Vert=\frac{1}{v}$. For $\forall\ j\neq r$, $\Vert H_{jj}^S \Vert=1,\Vert H_{jr}^S \Vert=\frac{1}{v}$. Moreover, $\Vert H^D\Vert=\Vert H^S\Vert=1$, and $D$, $S$, $\hat D$ and $\hat S$ are positive semi-definite matrices.
\end{lemma}
\begin{proof}
	First we can show that for $\forall j,r\leq v$, $H_{jr}^D=-\frac{1}{v\tilde{n}}\bm{1_{\tilde{n}}1_{\tilde{n}^T}}+\frac{c}{v\tilde{n}}\bm{I_c}$, where $\bm{I_c}$ is a $c\times c$ block matrix with diagonal matrices all equal to $\bm{1_{\frac{\tilde{n}}{c}}1_{\frac{\tilde{n}}{c}}^T}$. The matrix $-\frac{1}{v\tilde{n}}\bm{1_{\tilde{n}}1_{\tilde{n}^T}}$ contains exactly one non-zero eigenvalue, which equals to $-\frac{1}{v}$. Also, $\frac{c}{v\tilde{n}}\bm{I_c}$ has $c$ positive eigenvalues equal to $\frac{1}{v}$. Hence, we have $rank(H_{jr}^D)=c-1$, and all $c-1$ non-zero eigenvalues are equal to $\frac{1}{v}$. By the definition of spectral norm is the largest magnitude of the eigenvalues, we obtain
	\begin{equation*}
	\Vert H_{jr}^D \Vert=\frac{1}{v},\quad \forall j,r.
	\end{equation*}
	Similar analysis could be applied to $H^S$. According to fundamental linear algebra theories on block matrices, $rank(H_{jj}^S)=n$, with $\frac{\tilde n}{c}$ eigenvalues equal to $\frac{v-1}{v}$ and the rest $\frac{c-1}{c}n$ eigenvalues being 1. In addition, $rank(H_{jr}^S)=c$, and all eigenvalues equal $-\frac{1}{v}$. Consequently, we obtain
	\begin{equation*}
	\Vert H_{jj}^S \Vert=1,\quad \Vert H_{jr}^S \Vert=\frac{1}{v}.
	\end{equation*}
	
	\noindent\textbf{Spectrum of large matrices.}\hspace{0.1in}$H^D$ is a $v\times v$ block matrix with repeating blocks $H_{jr}^D$. Hence, it admits the form of Kronecker product,
	\begin{equation*}
	H^D=\bm{1_v1_v^T}\otimes H_{jr}^D.
	\end{equation*}
	Consequently, the spectrum of $H_D$ consist of $c-1$ eigenvalues equal to $\frac{1}{v}\cdot v=1$, and the rest all equal to 0. Therefore, $H^D$ is positive semi-definite ($i.e$ $H^D\succeq 0$). Recall the notation $K=diag(K_1,K_2,...,K_v)$, we have
	\begin{equation*}
	D=K^TH_DK\succeq 0,
	\end{equation*}
	since for $\forall x\in R^n$, $x^TK^TH^DKx=\tilde{x}^TH^D\tilde{x}\geq 0$. Define  $H_{off}=H_{jr}^S$ for $j\neq r$ as the off-diagonal block matrix of $H^S$. We have
	\begin{equation*}
	H^S=\bm{1_v1_v^T}\otimes H_{off}+diag_{v\times v}(I_{n\times n}).
	\end{equation*}
	The eigenvalues of $\bm{1_v1_v^T}\otimes H_{off}$, by previous analysis, are -1 with multiplicity $c$ and 0 with multiplicity $v\tilde n-c$. By adding diagonal block matrix of identities, $H^S$ has $c$ eigenvalues of 0 and all others equal to 1. Therefore, $S$ is also positive semi-definite.
\end{proof}

Lemma~\ref{norm lemma} summarizes the spectral property of covariance structure sub-matrices. In particular, it illustrates that the generalized eigen problem arise from KMvDA is definite, and thus the following analysis would be valid.

\subsubsection{A general perturbation bound}
As discussed in preliminaries, a feasible solution to (\ref{GEP3}) exists as long as the GEP is definite, which does not require $S$ to be invertible. We first consider this general situation. The next lemma is a modified version of Theorem 3 in~\cite{Proc:Lopez_ICML14}, which characterizes the kernel approximation error.

\begin{lemma}\label{error theorem}
	Suppose $X\subset \mathcal X^n$. Define linear approximation $\hat K_{n\times n}$ using $m$ random samples as (11). Then with probability $1-\eta$,
	\begin{align*}
	\Vert \hat K-K\Vert\leq &\frac{2n\log \frac{2n}{\eta}}{3m}+
	\frac{\sqrt{4n^2(\log\frac{2n}{\eta})^2+18mn\Vert K\Vert\log \frac{2n}{\eta}}}{3m}.
	\end{align*}
\end{lemma}

\begin{proof}
	We denote $F_{w_i}=[\dot f_{w_i}(x_1) \dots \dot f_{w_i}(x_n)]^T$, and define random matrices $Z_i=\frac{1}{m}(F_{w_i}F_{w_i}^T-K)$. By the unbiasedness of RFF's, we know that $EZ_i=0$. To bound $\Vert X_i\Vert$, we have $\Vert Z_i\Vert=\frac{1}{m}\Vert (F_{w_i}F_{w_i}^T-K)\Vert\leq \frac{2n}{m}$, due to triangle inequality and boundedness of $K$. In addition, we have
	\begin{align*}
	EZ_i^2&=\frac{1}{m^2}E[(F_{w_i}F_{w_i}^T-K)^2]\\
	&\leq \frac{1}{m^2}E[nF_{w_i}F_{w_i}^T-2F_{w_i}F_{w_i}^TK+K^2]\leq \frac{nK}{m^2}.
	\end{align*}
	The second line is due to the fact that $\Vert F_{w_i}^TF_{w_i}\Vert^2\leq n$. Thus,
	\begin{equation*}
	\sigma^2=\Vert\sum_{i=1}^mEZ_i^2\Vert\leq m\Vert EZ_i^2\Vert\leq \frac{n\Vert K\Vert}{m}.
	\end{equation*}
	Applying matrix Bernstein inequality (Theorem 5.4.1 in~\cite{Book:Vershynin_2018}),
	\begin{equation*}
	P\{\Vert \hat K-K\Vert\geq t\}\leq 2n\exp\big( -\frac{t^2/2}{n\Vert K\Vert/m+2nt/3m} \big).
	\end{equation*}
	Now taking the right-had-side to be equal to $\eta$, we derive a quadratic equation of $t$. Solving for this equation gives us the desired bound.
\end{proof}
It is worth mentioning that because of the correlated entries of $\hat K$, in general this bound cannot be reduced in the absence of more structural assumptions. Now we are ready to study the eigenspace perturbation caused by kernel approximation.

\begin{thm}
	For the GEP associated with KMvDA (i.e., (\ref{GEP3})), assume that $D$, $S$, $\hat D$ and $\hat S$ admit decompositions (\ref{decomp}) in the form of $M=diag(M_1,M_2)$ correspondingly. Let $\lambda(D,S)$ denote the set of eigenvalues of (\ref{GEP3})). Assume the Crawford number $\mathcal C(D,S)>0$, $\mathcal C(\hat D,\hat S)>0$, and there are $\alpha\geq0,\delta>0$ satisfying $\alpha+\delta\leq 1$, and a real number $\gamma$, such that
	\begin{equation}
	\begin{aligned}
	&\frac{|\gamma-\lambda_i|}{\sqrt{\gamma^2+1}\sqrt{\lambda_i^2+1}}\leq\alpha,\quad \forall \lambda_i\in\lambda(D_1,S_1),\\
	&\frac{|\gamma-\hat\lambda_i|}{\sqrt{\gamma^2+1}\sqrt{\hat\lambda_i^2+1}}\geq\alpha+\delta,\quad \forall \hat\lambda_i\in\lambda(\hat D_2,\hat S_2).\label{sep condition}
	\end{aligned}
	\end{equation}
	Denote $\Vert K^\star\Vert=\displaystyle\max_{i=1,...,v}\Vert K_i\Vert$, $\Vert \hat K^\star\Vert=\displaystyle\max_{i=1,...,v}\Vert \hat K_i\Vert$. Then the following inequality holds with probability $1-\eta$,
	\begin{align*}
	\Vert\sin\Theta\Vert\leq\frac{p(\alpha,\delta,\gamma)\Vert K^\star\Vert^2\xi_\eta}{\mathcal C(D,S) \mathcal C(\hat D,\hat S)}\cdot\frac{\Vert K^\star\Vert+\Vert \hat K^\star\Vert}{\delta},
	\end{align*}
where
	\begin{align}\notag
	p(\alpha,\delta,\gamma)=\frac{q(\gamma)[(\alpha+\delta)\sqrt{1-\alpha^2}+\alpha\sqrt{1-(\alpha+\delta)^2}]}{2\alpha+\delta}
	\end{align}
	with $q(\gamma)=2\sqrt 2$ for $\gamma\neq 0$ and $q(0)=2$. Also, we have
	\begin{align}\notag
 &\xi_\eta=\frac{2n\log \frac{2n/v}{1-(1-\eta)^{1/v}}}{3vm}+\\\notag
 &\frac{\sqrt{4(n/v)^2(\log\frac{2n}{1-(1-\eta)^{1/v}})^2+\frac{18}{v}mn\Vert K^\star\Vert\log \frac{2n/v}{1-(1-\eta)^{1/v}}}}{3m}
	\end{align}
	where $m$ is the number of random features.   \label{main}
\end{thm}

\begin{proof}
	By Theorem~\ref{error theorem}, with probability $1-\eta$, we have for $\forall i=1,...,v$,
		\begin{align*}
		&\Vert \hat K_i-K_i\Vert \leq \frac{2n\log \frac{2n/v}{1-(1-\eta)^{1/v}}}{3vm}+\\
		&	\frac{\sqrt{4(n/v)^2(\log\frac{2n}{1-(1-\eta)^{1/v}})^2+\frac{18}{v}mn\Vert K^\star\Vert\log \frac{2n/v}{1-(1-\eta)^{1/v}}}}{3m}.
		\end{align*}
	Denote this event $\Omega$. In this event, we have
	\begin{align*}
	\Vert D-\hat D\Vert&=\Vert KH^DK-\hat KH^D\hat K\Vert \\
	&=\Vert KH^DK-\hat KH^DK+\hat KH^DK-\hat KH^D\hat K\Vert \\
	&=\Vert (K-\hat K)H^DK+\hat KH^D(K-\hat K)\Vert \\
	&\leq\Vert K-\hat K\Vert\Vert H^D\Vert\Vert K\Vert+\Vert\hat K\Vert\Vert H^D\Vert\Vert K-\hat K \Vert\\
	&=\Vert K-\hat K\Vert(\displaystyle\max_{i=1,...,v}\Vert K_i\Vert+\displaystyle\max_{i=1,...,v}\Vert \hat K_i\Vert) \Vert,
	\end{align*}
	where we recall that $K=diag(K_1,...,K_v)$ and $\hat K=diag(\hat K_1,...,\hat K_v)$. The last line holds because $\Vert H^D\Vert=1$ and $K$, $\hat K$ are both diagonal block matrix. Therefore,
	$$\Vert K\Vert=\Vert K^\star\Vert,\ \Vert \hat K\Vert=\Vert \hat K^\star\Vert.$$
	It is easy to check that $\Vert S-\hat S\Vert \leq\Vert K-\hat K\Vert(\displaystyle\max_{i=1,...,v}\Vert K_i\Vert+\displaystyle\max_{i=1,...,v}\Vert \hat K_i\Vert) \Vert$ analogously using same argument. Moreover, by sub-multiplicity of operator norms, we have
	
	\begin{align*}
	\sqrt{\Vert D^2+(S+\epsilon I)^2\Vert}&\leq \sqrt{\Vert D^2\Vert+\Vert (S+\epsilon I)^2\Vert} \\
	&=\sqrt{\Vert KH^DK\Vert^2+\Vert KH^SK\Vert^2}\\
	&\leq \sqrt{\Vert K\Vert^4+(\Vert K\Vert^2)^2}\\
	&\leq \sqrt{2(\Vert K\Vert^2)^2}=\sqrt 2(\Vert K^\star\Vert^2),
	\end{align*}
	since $\Vert H^D\Vert=\Vert H^S\Vert=1$. Because $Z_1$ is orthogonal, we have
	$$\Vert (D-\hat D)Z_1\Vert\leq\Vert D-\hat D\Vert\Vert Z_1\Vert=\Vert D-\hat D\Vert,$$
	and same inequality holds for $S$. Hence we have
	\begin{align*}
	    \sqrt{\Vert (D-\hat D)Z_1\Vert^2+\Vert S-\hat S)Z_1\Vert^2}  \leq \xi_\eta(\Vert K^\star\Vert+\Vert \hat K^\star\Vert).
	\end{align*}
	Putting all parts together and using Theorem 2.1 from~\cite{Article:Sun_1983}, we get the desired bound.	
\end{proof}

Condition (\ref{sep condition}) characterizes the separation of the generalized eigenvalues, where the eigengap can be interpreted in terms of chordal distance, defined by (\ref{chordal}). Since the generalized eigenvalues are invariant of scale, we may force them on a unit semicircle in the upper plane. Note that in our problem, the generalized eigenvalues are in the form $(\lambda_i,1)$. Hence, we can scale each eigenvalue to $(s_i,t_i)\triangleq(\frac{\lambda_i}{\sqrt{\lambda_i^2+1}},\frac{1}{\sqrt{\lambda_i^2+1}})$. For any two pairs, we have

$$\sin((s_i,t_i),(\tilde s_i,\tilde t_i))=\frac{|\lambda_i-\tilde{\lambda_i}|}{\sqrt{\lambda_i^2+1}\sqrt{\tilde\lambda_i^2+1}}=\rho((s_i,t_i),(\tilde s_i,\tilde t_i)).$$

That is, the chordal distance between two eigenvalue pairs is the sine between the two rays with slopes $\frac{1}{\lambda_i}$ and $\frac{1}{\tilde \lambda_i}$ extended from the origin. Now we can translate (\ref{sep condition}) into angles (defined anti-clockwise): there exist a real number $\gamma$, $\alpha\geq 0$, $\delta>0$ and $\alpha+\delta\leq 1$, such that
\begin{align}\notag
&\displaystyle\max_{\lambda_i\in\lambda(D_1,S_1)}\sin((\lambda_i,1),(\gamma,1))\leq \alpha\triangleq \sin\theta,\\\notag
&\displaystyle\min_{\tilde\lambda_i\in\lambda(\hat D_1,\hat S_1)}\sin((\tilde\lambda_i,1),(\gamma,1))\geq \alpha+\delta\triangleq \sin\tilde\theta.
\end{align}
Define $\theta_g=\displaystyle\min_{{\lambda_i\in\lambda(D_1,S_1)},{\tilde\lambda_i\in\lambda(\hat D_1,\hat S_1)}}\sin((\lambda_i,1),(\tilde\lambda_i,1))$ as the gap between eigenvalue sets $\lambda(D_1,S_1)$ and $\lambda(\hat D_1,\hat S_1)$. It is easy to check that $\theta_g=\theta-\tilde{\theta}$, and
\begin{align*}
  \sin(\theta_g)&=\sin(\tilde{\theta})\cos(\theta)-\cos(\tilde{\theta})\sin(\theta)\\
  &\geq(\alpha+\delta)\sqrt{1-\alpha^2}-\alpha\sqrt{1-(\alpha+\delta)^2}>0,
\end{align*}
which implies that two sets of eigenvalues are well separated.

\subsubsection{Perturbation of regularized problem}
In practice, a regularization term is often added to GEP to handle singularity and make the system more stable and theoretically justifiable. Consider the following regularized GEP,
\begin{equation}
    Dw=(S+\epsilon I)w, \label{GEP_reg}
\end{equation}
with $\epsilon>0$ a small constant. The problem is guaranteed to be definite, since $(S+\epsilon I)$, by Lemma~\ref{norm lemma}, now becomes positive definite. More importantly, the invertibility of  $(S+\epsilon I)$ allows us to transform (\ref{GEP_reg}) into an SEP.

\begin{thm} \label{reg_theo}
Let $\lambda_1\geq \lambda_2\geq \dots\geq\lambda_n$ denote eigenvalues of $(S+\epsilon I)^{-1}D$, and $\hat\lambda_1\geq \hat\lambda_2\geq \dots\geq\hat\lambda_n$ be the eigenvalues of $(\hat S+\epsilon I)^{-1}\hat D$. Assume $\lambda_l-\hat\lambda_{l+1}=\delta>0$, then with probability $1-\eta$,
\begin{equation*}
    \Vert \sin\Theta\Vert\leq \frac{\xi_\eta}{\delta}\cdot\bigg\{ C\frac{\Vert K^*\Vert^2(\Vert K^*\Vert+\Vert \hat K^*\Vert)}{\epsilon^2}+\frac{(\Vert K^*\Vert+\Vert \hat K^*\Vert)}{\epsilon} \bigg\},
\end{equation*}
where $C=\frac{1+\sqrt 5}{2}$. $\Vert K^*\Vert,\Vert \hat K^*\Vert$ and $\xi_\eta$ are defined in Theorem~\ref{main}.
\end{thm}

\begin{proof}
(of Theorem 3) Since $(S+\epsilon I)$ is invertible, we may consider the SEP $(S+\epsilon I)^{-1}Dw=\lambda w$. We have
\begin{align}
    &\Vert (S+\epsilon I)^{-1}D-(\hat S+\epsilon I)^{-1}\hat D \Vert \nonumber\\
    =&\Vert [(S+\epsilon I)^{-1}-(\hat S+\epsilon I)^{-1}]D+(\hat S+\epsilon I)^{-1}(D-\hat D)\Vert \nonumber\\
    \leq& \Vert [(S+\epsilon I)^{-1}-(\hat S+\epsilon I)^{-1}]D\Vert \nonumber\\
    &\hspace{0.8in} +\Vert (\hat S+\epsilon I)^{-1}(D-\hat D)\Vert \nonumber\\
    \overset{(i)}{\leq}& C\frac{\Vert K^*\Vert^2(\Vert K^*\Vert+\Vert \hat K^*\Vert)\xi_\eta}{\epsilon^2}+\frac{(\Vert K^*\Vert+\Vert \hat K^*\Vert)\xi_\eta}{\epsilon}, \nonumber
\end{align}
where $C=\frac{1+\sqrt 5}{2}$. Here $(i)$ is induced by Theorem 4.1 in~\cite{Article:Wedin_1973}. Since $(S+\epsilon I)$ is positive definite and symmetric, $(S+\epsilon I)^{-1}$ is also symmetric and positive definite. Given that $D$ is symmetric and positive semi-definite, we know that $(S+\epsilon I)^{-1}D$ is similar to a symmetric PSD matrix,
\begin{align*}
    &(S+\epsilon I)^{1/2}[(S+\epsilon I)^{-1}D](S+\epsilon I)^{-1/2}\\
    &\hspace{0.8in} =(S+\epsilon I)^{-1/2}D(S+\epsilon I)^{-1/2}.
\end{align*}
Hence, the eigenvalues of $(S+\epsilon I)^{-1}D$ are all real and non-negative. Therefore, the eigenvalues is equivalent to singular values. The proof is then complete using the classic sin$\Theta$ Theorem from~\cite{Article:Wedin_1972}.
\end{proof}

From Theorem~\ref{main} and Theorem~\ref{reg_theo}, we know that for both the original and regularized GEP, adopting linearized kernels could approximate the eigenspace of using exact kernel matrices, with a sufficiently large number of random features. This provides a theoretical support for the usage of RFF's in KMvDA.

\subsubsection{Comparison to Randomized CCA}
In~\cite{Proc:Lopez_ICML14}, the authors propose randomized CCA (RCCA), which also solves a GEP in the form of $Ax=\lambda (B+\epsilon I)x$. However, it turns out that the problem is very different. More specifically,
\begin{itemize}
    \item RCCA only involves two views, while KMvDA may include multiple views.

    \item The covariance matrices in RCCA is much simpler (block diagonal and linear in $K$), while for KMvDA the formulation is more sophisticated and quadratic in $K$.

    \item We consider both the regularized problem and the general case of definite eigen problem without regularization, while~\cite{Proc:Lopez_ICML14} only studies the formulation with regularization.
\end{itemize}

\section{Experiments}

In this section, we present empirical results that illustrate the performance of KMvDA and linearized KMvDA using random Fourier features. The major goal is to show 1) KMvDA improves linear MvDA, and 2) randomized KMvDA is able to well approximate the performance of KMvDA with sufficient number of RFF's.

\subsection{Datasets}

We test our algorithms on 3 popular datasets for multi-view learning research and applications. All datasets are publicly available.

\textbf{Heterogeneous Face Biometrics (HFB)} database has 100 persons in total, with 4 composed of visual (VIS) and 4 near infrared (NIR) face images for each person. This gives us a 2-view classification problem. For each view, we have 400 examples in total from 100 different people. We use the first 65 persons for training and the remaining 35 persons for testing. Each example is a $32\times 32$ image, which is transformed into 1024 features.

\textbf{CUHK Face Sketch FERET (CUFSF)} database is designed for research on face sketch synthesis and face sketch recognition. It includes 1194 persons (i.e., categories) from the FERET database. An example is given in Figure~\ref{fig:CUFSF-example}. This dataset contains two views: 1) face photo with lighting variation, and 2) sketch with shape exaggeration drawn by an artist when viewing this photo, both with dimensionality 5120. We use the first 650 examples as training set and the rest 544 examples for testing.

We use \textbf{Multi-PIE} dataset to test the performance of KMvDA on dealing with multiple views and larger sample size. The whole dataset contains more than 750,000 face images of 337 people, under different poses and from distinct views. In our experiment, we choose 7 different views (left $45^\circ$, $30^\circ$, $15^\circ$, frontal, right $15^\circ$, $30^\circ$, $45^\circ$), three facial expressions (smile, neutral, disgust), and no flush illumination as the evaluation data. Each example is a 5,120 dimensional vector. This subset is divided into two parts: the images from the first 248 subjects with 4 randomly selected images under each pose of each person are utilized as training data and the images from the rest are utilized as test data.

\subsection{Parameters and Performance Evaluation}

\textbf{Kernels.}\hspace{0.1in} There is no tuning parameter for linear kernel. For RBF kernel, we fine-tune the parameter $\sigma$ over a fine grid in the range of $\{0.001, 100\}$. The number of random Fourier features are chosen to be $m=\{2^6,2^7,...,2^{15}\}$ for each view. We set $\sigma$ for RFF's the same as fine tuned parameter value for RBF kernel to compare the approximation effectiveness of linearized methods. RFF vectors are normalized to have unit norm.

\vspace{0.05in}
\noindent\textbf{Evaluation.}\hspace{0.1in}  We mainly use the classification test accuracy to evaluate the model performance. We denote ``$v_2$-$v_1$'' when using training examples from view $v_1$ to classify test examples from view $v_2$. The metric we use is the rank-1 recognition rate, which is the highest test accuracy among all parameter $\sigma$ and projection dimensionality $d$.

\begin{table}
	\caption{Results of rank-1 recognition rate (\%) of different kernels.	\vspace{-0.15in}}
	\label{Tab1}
    \centering
		\scalebox{1.1}[1]{
		\begin{tabular}{l|l|llll}\hline
			&              & \textbf{Linear} & \textbf{RBF}   & \textbf{RFF $\dot f$}  \\
			\hline
			\multirow{2}{*}{\makecell[l]{\textbf{HFB }}}   & NIR-VIS      & 56.4     & 59.3   & \textbf{60.7} \\
			& VIS-NIR      & 47.2    & \textbf{51.4}   & \textbf{51.4}\\
			\hline
			\multirow{2}{*}{\makecell[l]{\textbf{CUFSF }}} & Photo-Sketch & 45.2          & \textbf{52.4}       & 51.0 \\
			& Sketch-Photo & \textbf{52.6}    & 52.2  & 52.4 \\
			\hline
			\textbf{Multi-PIE} & Avg. Accuracy & 93.6   & \textbf{94.8}   & \textbf{94.8}  \\\hline
		\end{tabular}}
\end{table}

\subsection{Experiment Results}

\textbf{Overall performance.}\hspace{0.1in}  Table~\ref{Tab1} summarizes the rank-1 recognition rate of different approaches on HFB and CUFSF datasets, and the average rank-1 recognition rate among all 7 views for Multi-PIE dataset. As we can see, RBF kernel significantly outperforms linear kernel in almost all cases. In addition, the accuracy of using linearized approximation is very close to that of using RBF kernel directly, sometimes even slightly better.

\begin{figure}[h]
	\begin{center}
		\mbox{\hspace{-0.2in}
			\includegraphics[width=1.85in]{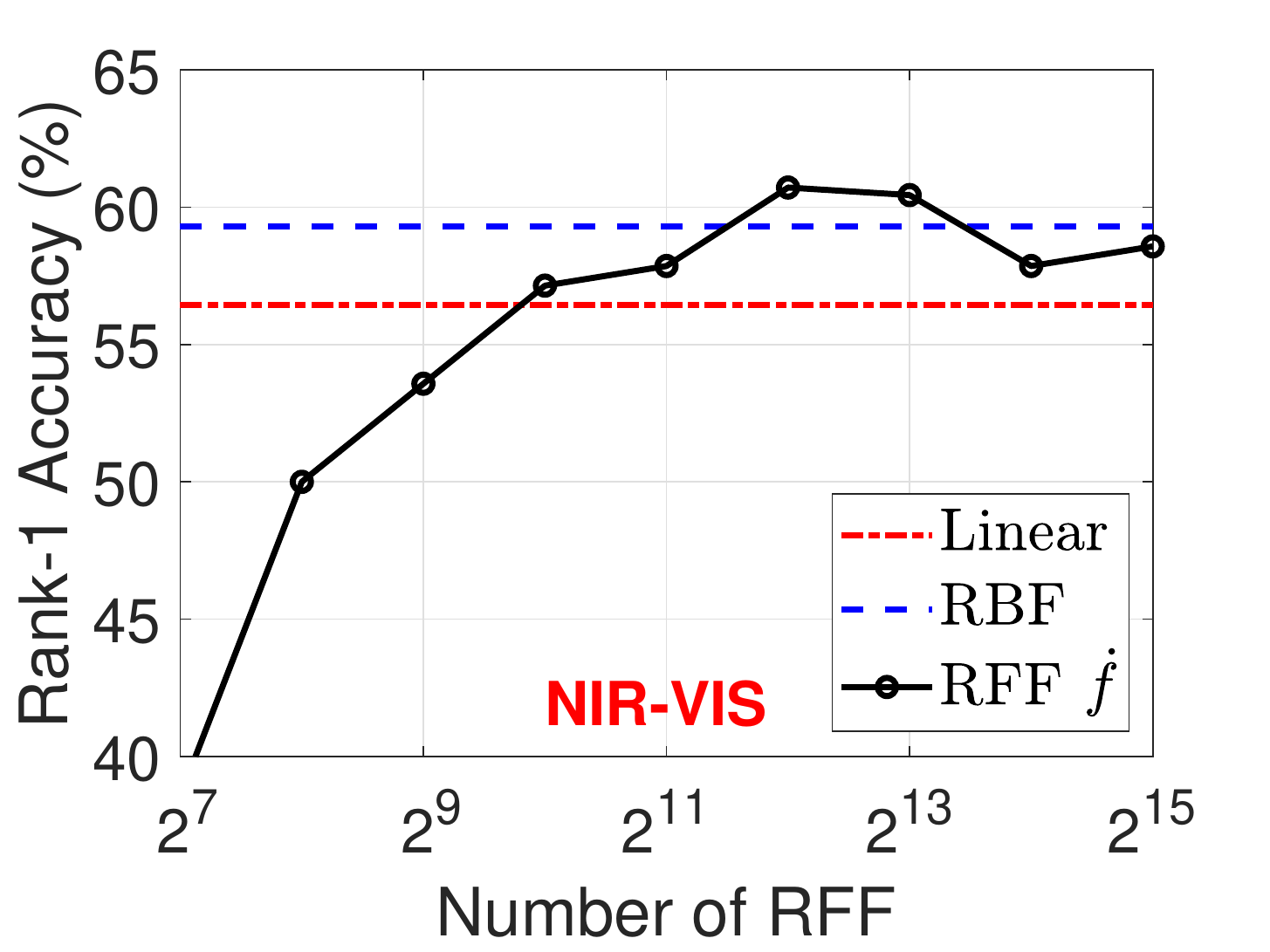}\hspace{-0.12in}
			\includegraphics[width=1.85in]{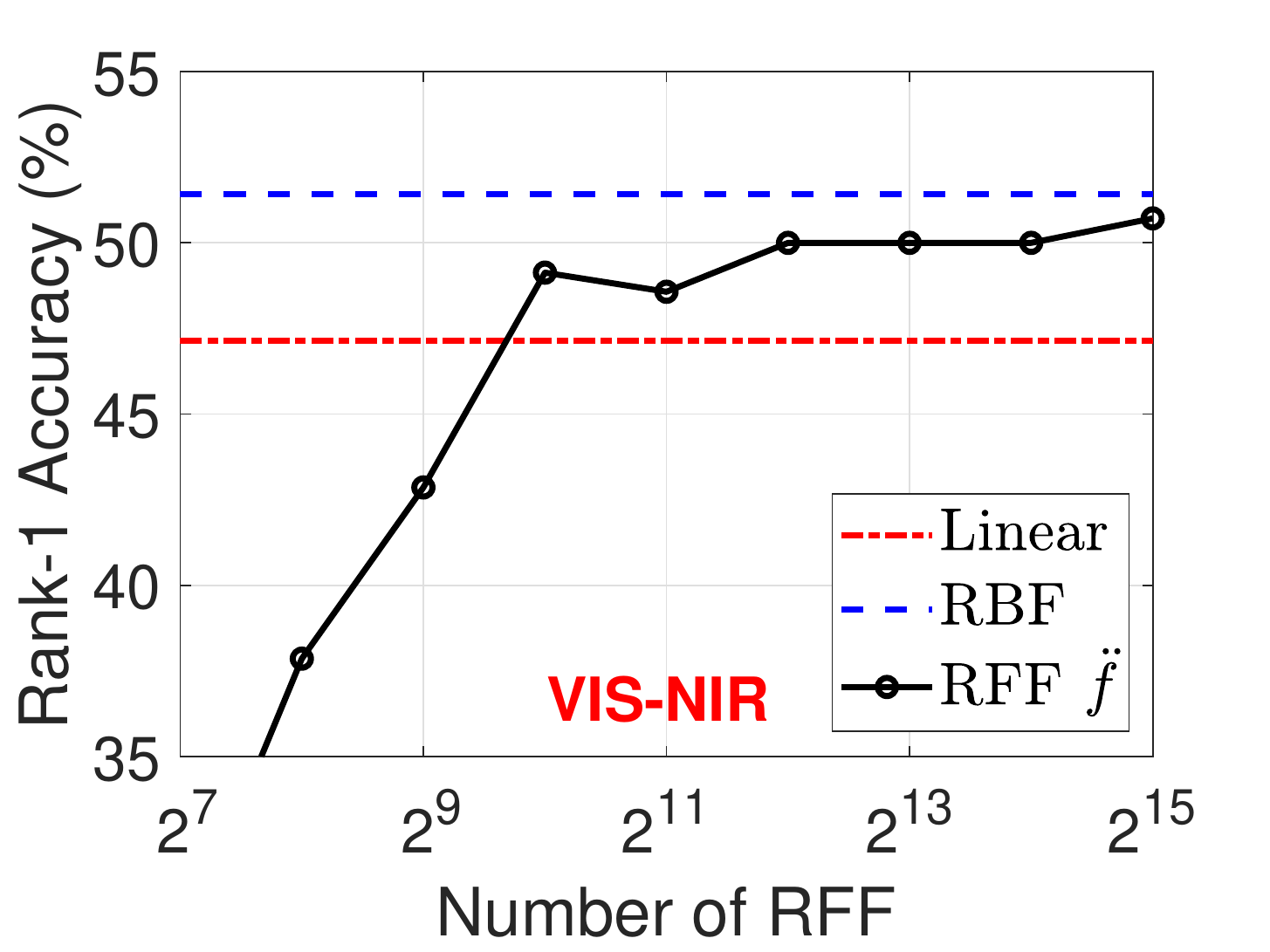}
		}
		\mbox{\hspace{-0.2in}
			\includegraphics[width=1.85in]{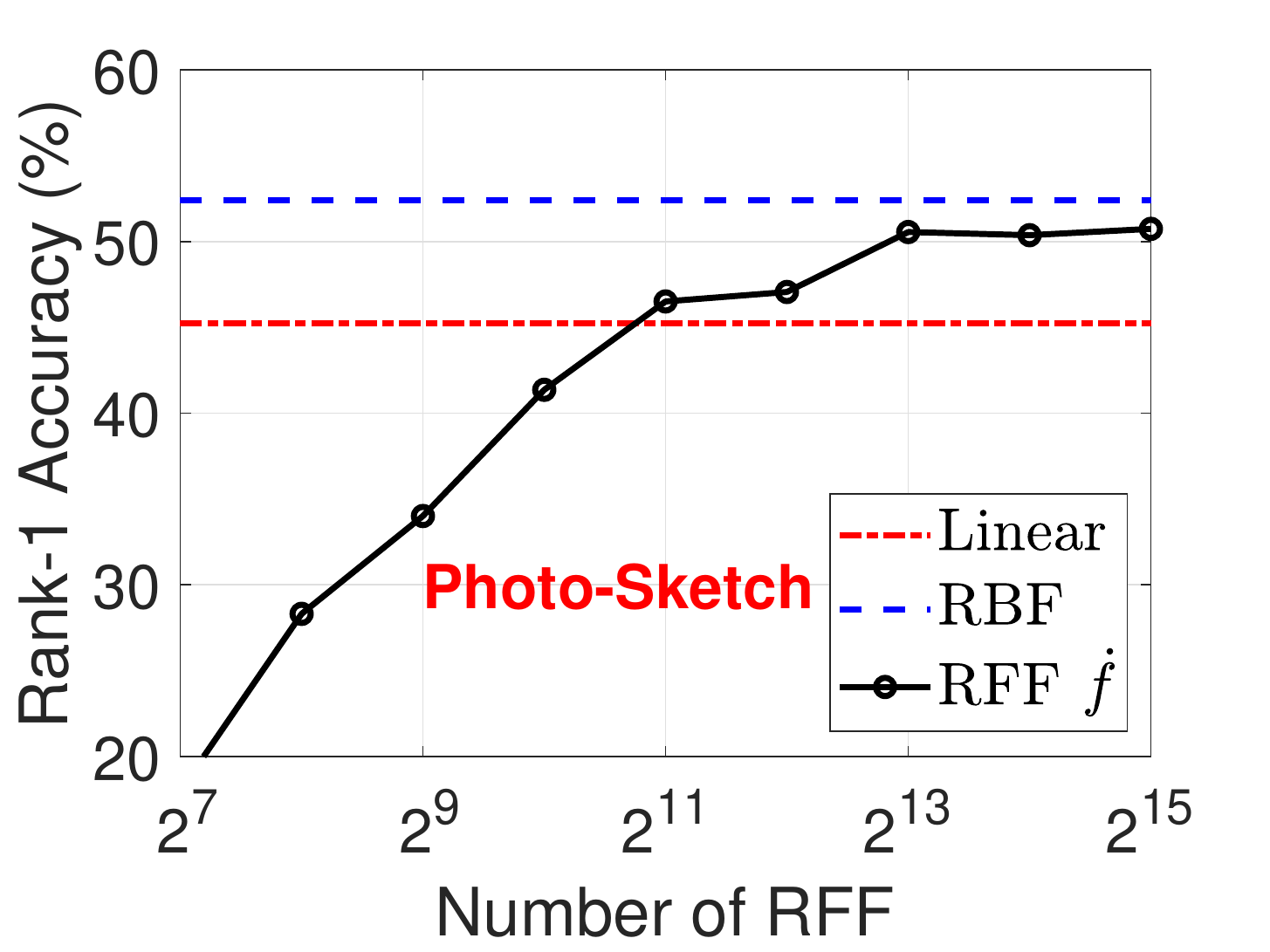}\hspace{-0.12in}
			\includegraphics[width=1.85in]{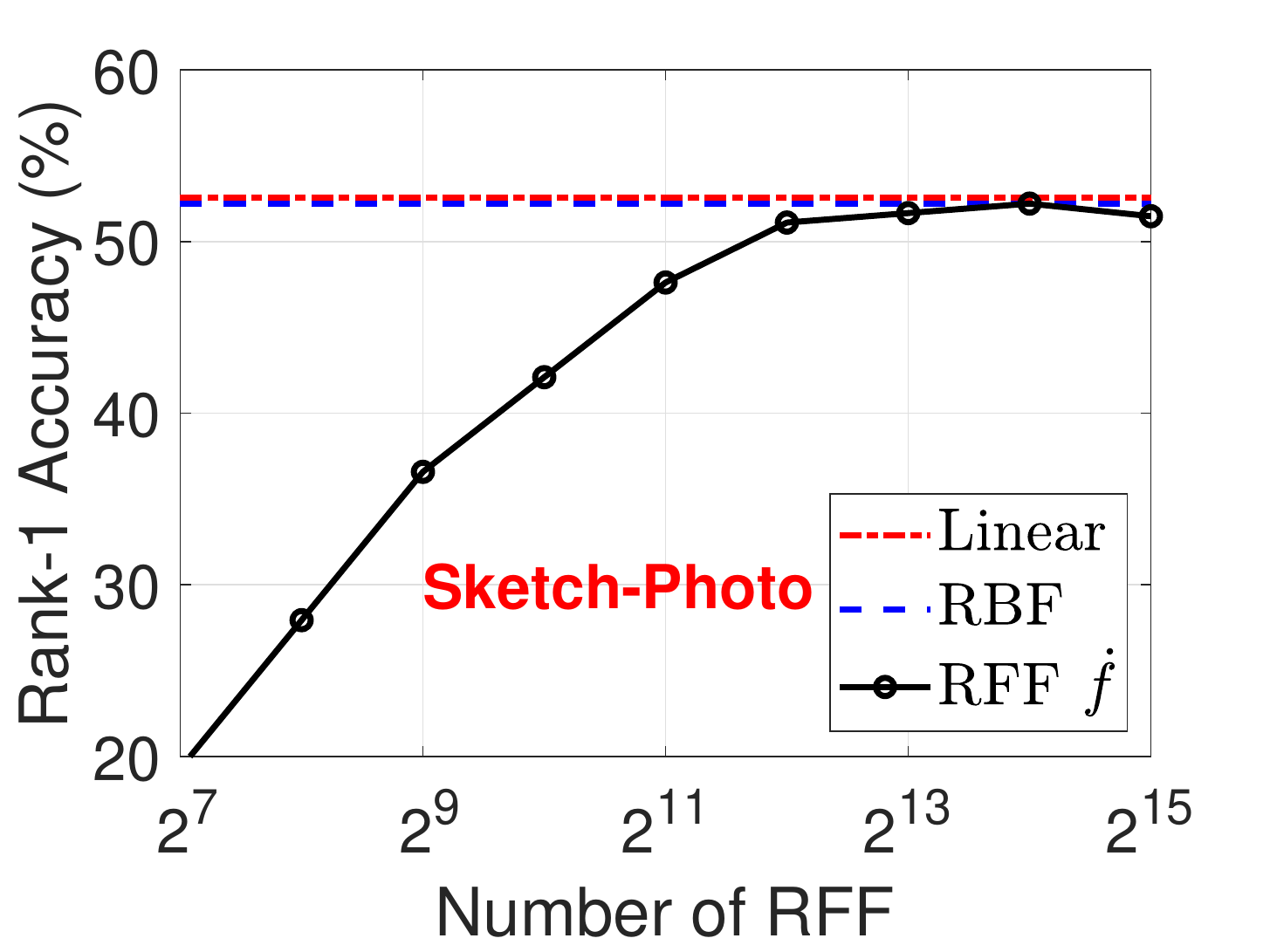}
		}
	\end{center}
 	\vspace{-0.2in}
	\caption{RFF's: rank-1 recognition rate vs. number of random Fourier features. The upper panel is for HFB dataset and the lower panel is for CUFSF dataset.}
	\label{figure1}
\end{figure}

\noindent\textbf{Number of features.}\hspace{0.1in}  In Figure~\ref{figure1}, we plot the highest test accuracy against different $m$, the number of random features. For HFB dataset, the recognition rate becomes stable at around $m=2^{11}$. For CUFSF and Multi-PIE (Figure~\ref{figure3}) dataset, this number is between $2^{12}$ to $2^{13}$. This is consistent with the observation in~\cite{Proc:Buazuavan_ECCV12} that a few thousands of RFF's are often required in order to provide good approximation.

\begin{figure}[h!]
	\begin{center}
		\mbox{
			\includegraphics[width=1.85in]{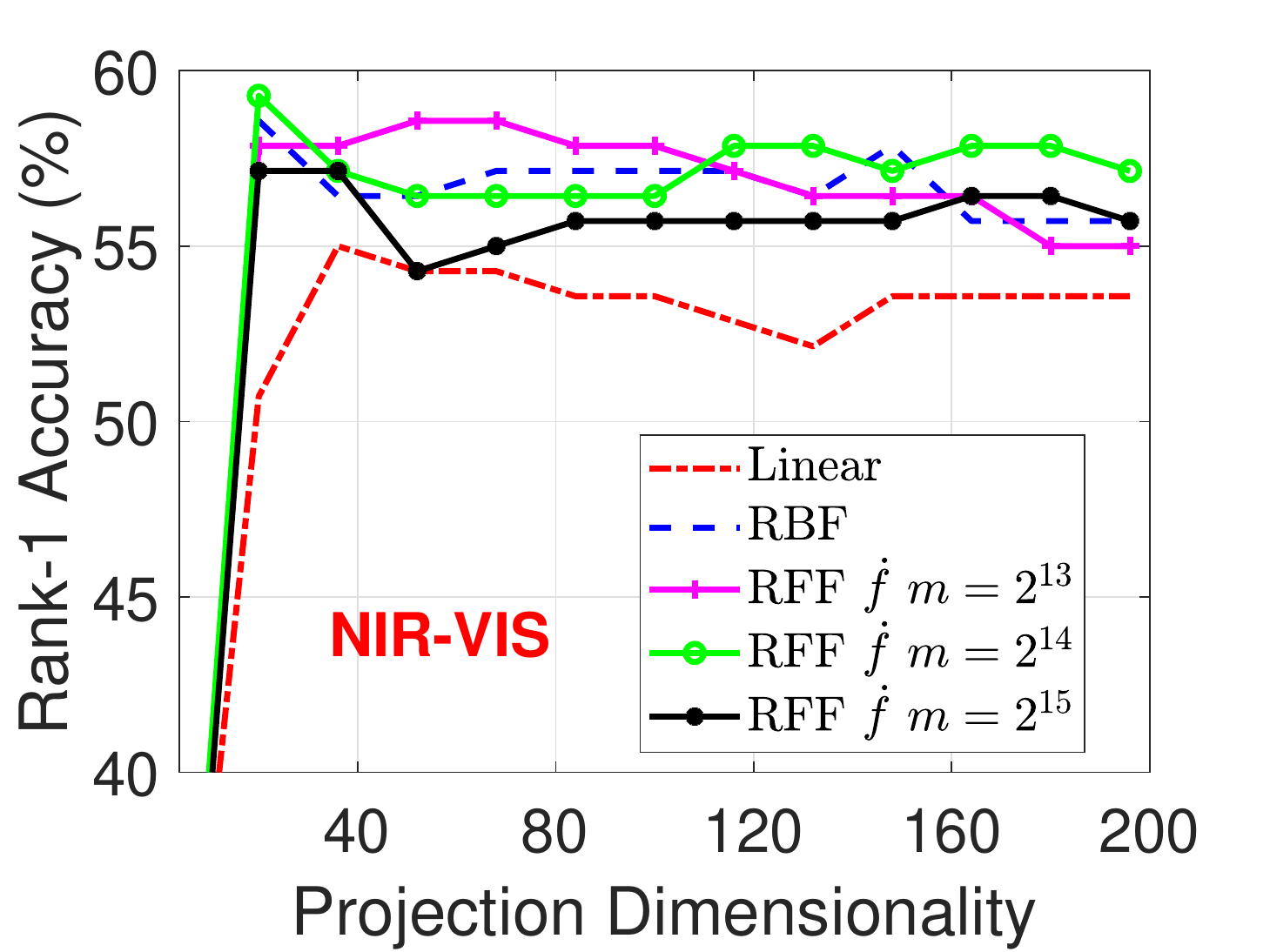}\hspace{-0.12in}
			\includegraphics[width=1.85in]{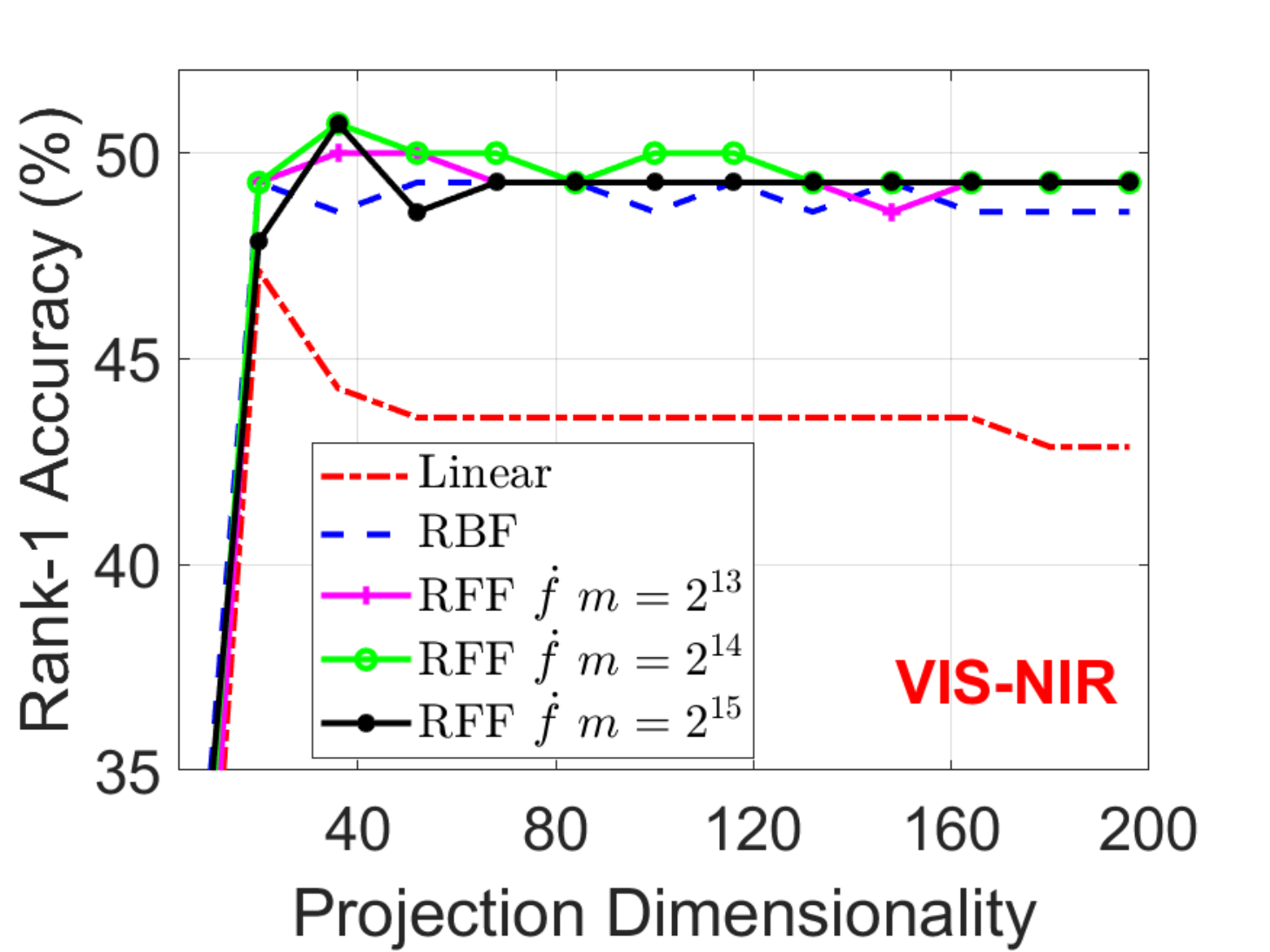}
		}
		\mbox{
			\includegraphics[width=1.85in]{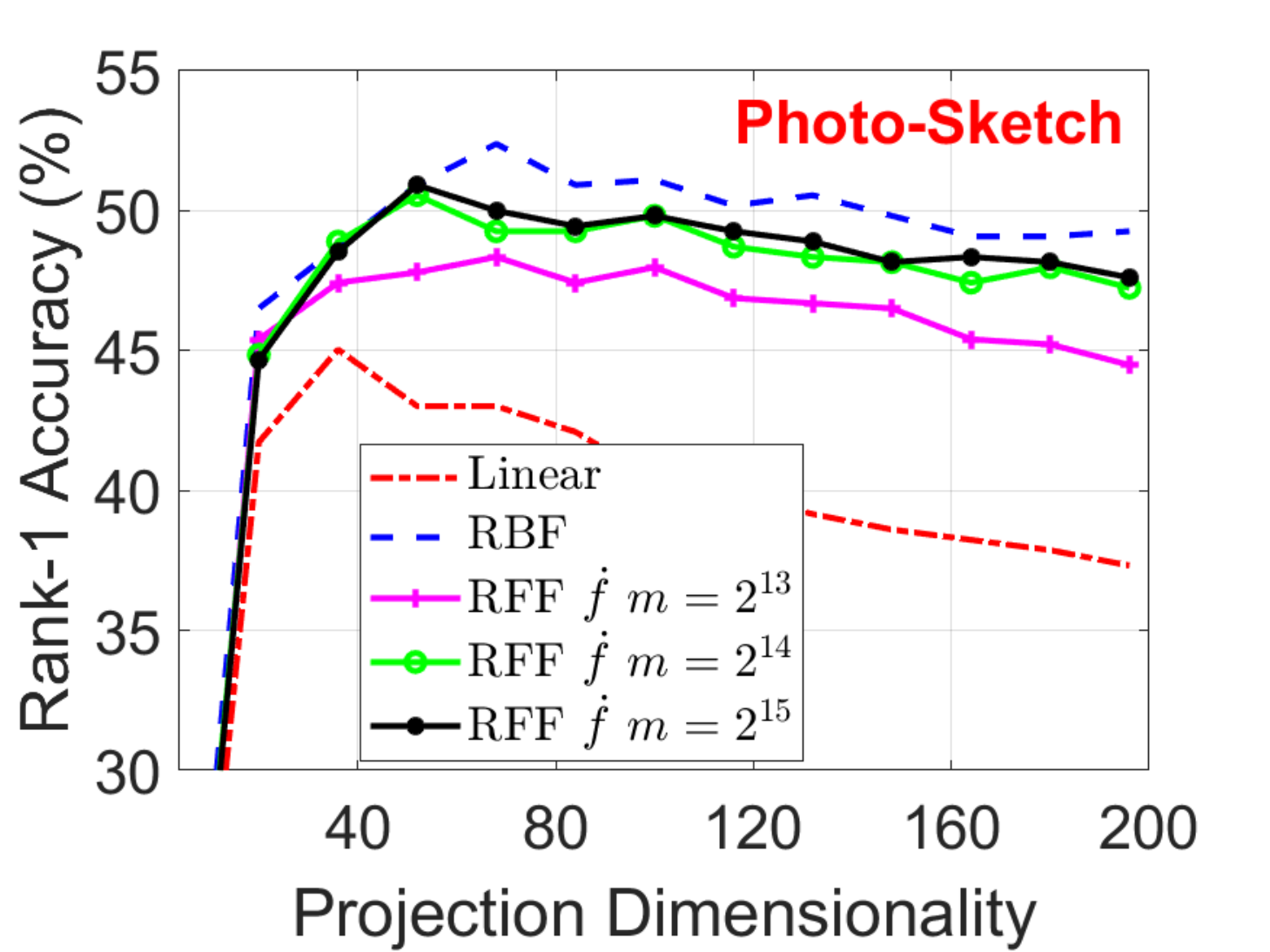}\hspace{-0.12in}
			\includegraphics[width=1.85in]{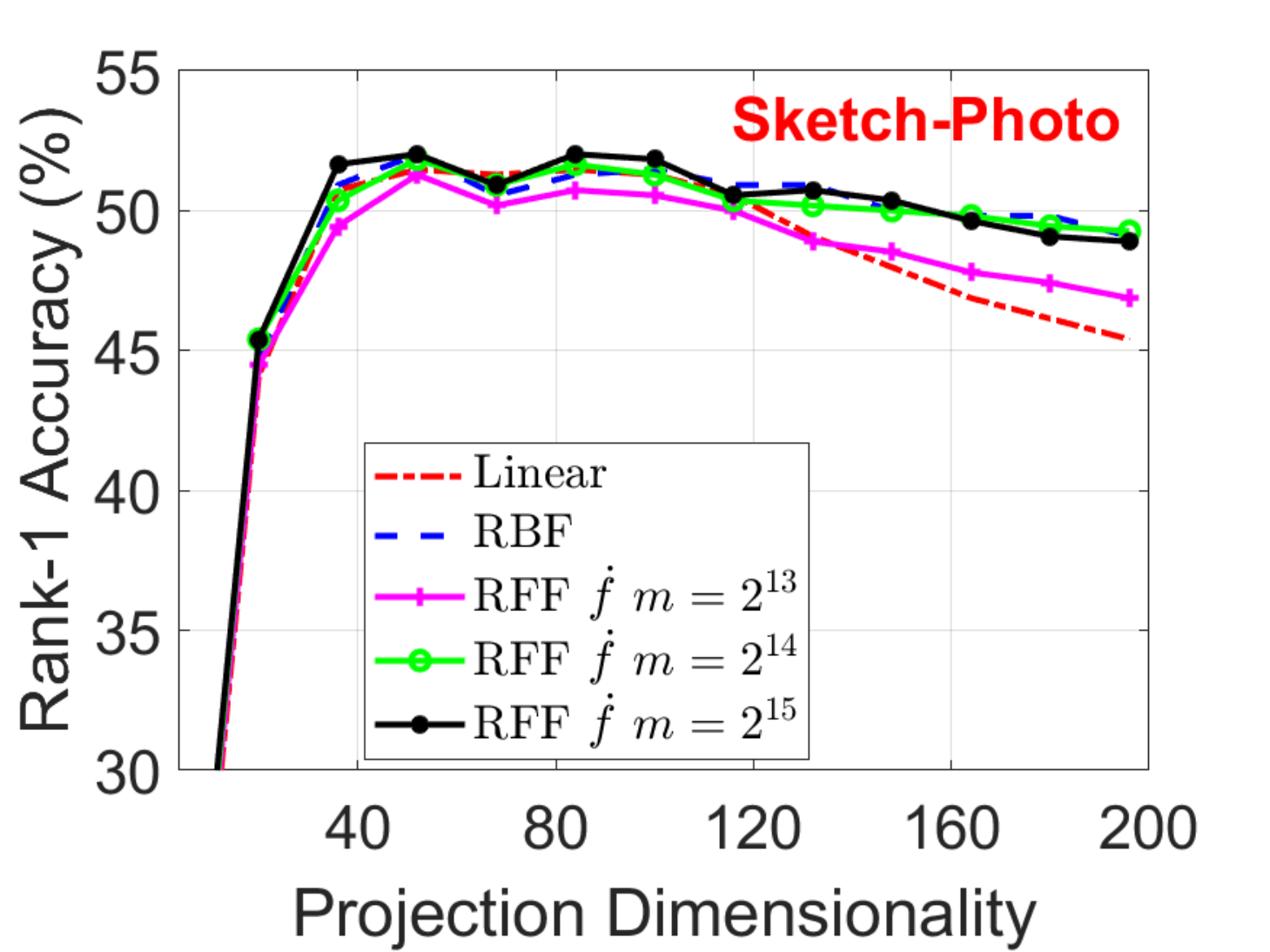}
		}
	\end{center}
 	\vspace{-0.2in}
	\caption{Linear kernel, RBF kernel and RFF's: rank-1 recognition rate vs. projection dimensionality. Upper panel: HFB. Lower panel: CUFSF.}
	\label{figure2} 
\end{figure}

\clearpage


\noindent\textbf{Number of projections.}\hspace{0.1in}  Figure~\ref{figure2} shows the rank-1 accuracy against the subspace dimensionality $l$. We observe for all cross-views, the performance of KMvDA stabilizes after the dimensionality reaches 50, which appears to be a good recommendation in practice. Also, adding more projection directions may deteriorate the test accuracy of linear kernel, since we observe significant decrease in recognition rate in all figures after $l=50$. In this sense, RBF kernel (as well as RFF's) is much more robust.\\

\noindent\textbf{Multi-PIE dataset.}\hspace{0.1in}  Tables~\ref{multi-pie linear},~\ref{multi-pie RBF}, and~\ref{multi-pie RFF2} demonstrate the best recognition rate among all views of Multi-PIE dataset. Here gallery means training view, and probe refers to test view. We see that RBF kernel improves the accuracy on almost all cross-views. The pair $(0,-45^{\circ})$ and $(0,45^{\circ})$ are most challenging tasks since the front face is most different from the face seen from $\pm 45^{\circ}$ angle. For these cross-views, RBF can increase the accuracy by around 5\%. Figure~\ref{figure3} shows the results on this cross-view. Figure~\ref{figure4} plots the average accuracy among all pair of views, which again confirms the convergence since the curves of RBF and RFF's almost overlap.

\begin{table}
	\caption{Multi-PIE: Linear, rank-1 recognition rate (\%).\vspace{-0.1in}}
	\label{multi-pie linear}
	\centering
	\scalebox{1}[1]{
		\begin{tabular}{l|lllllll}
			\toprule
			\multirow{2}{*}{Probe}   & \multicolumn{7}{c}{Gallery}      \\
			&$-45^{\circ}$ &$-30^{\circ}$ &$-15^{\circ}$ &$0^{\circ}$ &$15^{\circ}$  &$30^{\circ}$ & $45^{\circ}$   \\
			\hline
			-$45^\circ$  & -  & 97.77 & 93.63 & 87.58 & 86.26   &97.13 & 98.33  \\
			-$30^\circ$  & 97.77   & -  & 96.13  & 96.50&89.81  & 94.46 & 96.18 \\
			-$15^\circ$    & 95.70 & 99.04   & - & 99.36 &92.99  &95.86  & 93.27 \\
			$0^\circ$   & 88.85  & 96.82  & 97.54  & -  &90.45  & 91.85 & 87.22  \\
			$15^\circ$ & 90.45   & 87.93  & 90.76 & 90.88   & - &97.77  & 97.98 \\
			$30^\circ$ & 98.41  & 92.36  & 92.99  & 91.40  &97.77  & -  &99.21 \\
			$45^\circ$ & 98.73   & 94.90  & 93.63 & 88.12 &95.94 & 98.09 &-\\
			\hline
	\end{tabular}}
\end{table}

\begin{table}
	\caption{Multi-PIE: RBF, rank-1 recognition rate (\%).\vspace{-0.1in}}
	\label{multi-pie RBF}
	\centering
	\scalebox{1}[1]{
		\begin{tabular}{l|lllllll}
			\toprule
			\multirow{2}{*}{Probe}   & \multicolumn{7}{c}{Gallery}      \\
			&$-45^{\circ}$ &$-30^{\circ}$ &$-15^{\circ}$ &$0^{\circ}$ &$15^{\circ}$  &$30^{\circ}$ & $45^{\circ}$   \\
			\hline
			-$45^\circ$   & -     & 98.73 & 95.86 & 93.31   & 91.40 & 98.41  & 99.04 \\
			-$30^\circ$   & 98.09  & -     & 97.45 & 97.77   & 93.00 & 96.82 & 98.09  \\
			-$15^\circ$   & 97.77 & 99.04 & -     & 99.36   & 93.95 & 96.50 & 96.50 \\
			$0^\circ$ & 91.08 & 97.13 & 98.41  & -       & 92.36 & 94.27 & 90.76 \\
			$15^\circ$  & 92.68 & 91.08 & 92.04 & 93.31   & -     & 98.73 & 99.04 \\
			$30^\circ$  & 97.45 & 95.22 & 93.63 & 93.95   & 98.41 & -     & 99.04 \\
			$45^\circ$  & 99.04 & 97.77 & 93.63 & 90.13   & 97.45  & 99.36 & -\\
			\hline
	\end{tabular}}
\end{table}

\begin{table}
	\caption{Multi-PIE: RFF $\dot f$, rank-1 recognition rate (\%).\vspace{-0.1in}}
	\label{multi-pie RFF2}
	\centering
	\scalebox{1}[1]{
		\begin{tabular}{l|lllllll}
			\toprule
			\multirow{2}{*}{Probe}   & \multicolumn{7}{c}{Gallery}      \\
			&$-45^{\circ}$ &$-30^{\circ}$ &$-15^{\circ}$ &$0^{\circ}$ &$15^{\circ}$  &$30^{\circ}$ & $45^{\circ}$   \\
			\hline
			-$45^\circ$   & -     & 98.43 & 95.86 & 92.36   & 92.36 & 98.41  & 99.04 \\
			-$30^\circ$   & 98.09  & -     & 97.45 & 97.45   & 93.00 & 96.82 & 97.77  \\
			-$15^\circ$   & 98.09 & 99.04 & -     & 99.36   & 94.59 & 96.50 & 95.86 \\
			$0^\circ$ & 90.76 & 97.13 & 98.41  & -       & 92.04 & 93.95 & 91.08 \\
			$15^\circ$  & 92.68 & 91.40 & 92.68 & 92.68   & -     & 99.04 & 99.04 \\
			$30^\circ$  & 97.77 & 94.90 & 93.95 & 93.31   & 98.41 & -     & 99.04 \\
			$45^\circ$  & 99.04 & 97.45 & 93.95 & 89.81   & 97.45  & 99.36 & -\\
			\hline
	\end{tabular}}
\end{table}


\begin{figure}[t]
	\begin{center}
		\mbox{\hspace{-0.2in}
			\includegraphics[width=1.9in]{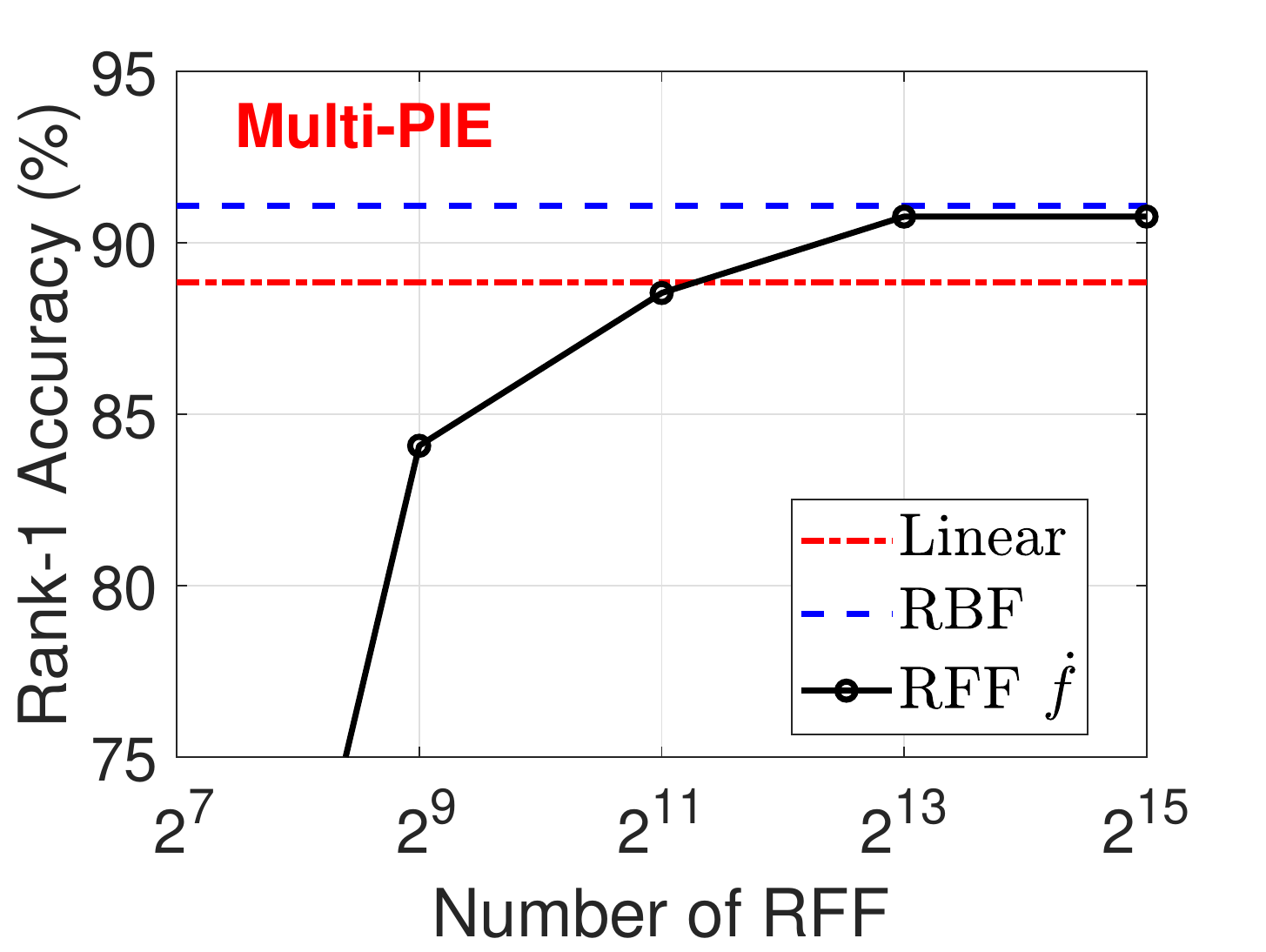}\hspace{-0.1in}
			\includegraphics[width=1.9in]{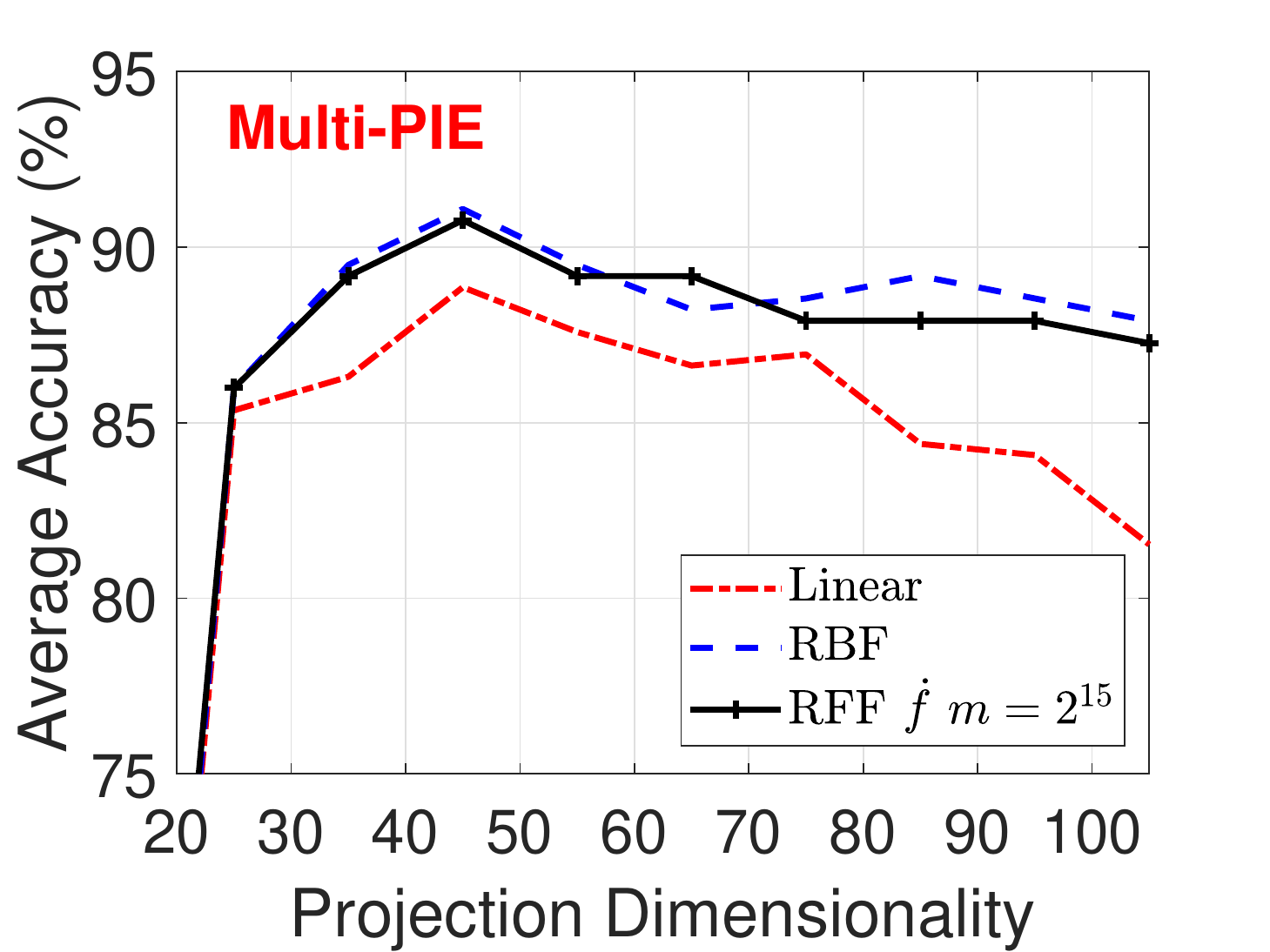}
		}
	\end{center}
	\vspace{-0.2in}
	\caption{Multi-PIE dataset $-45^\circ\rightarrow 0^\circ$ cross-view. Left panel: Accuracy vs. number of RFF's. Right panel: Accuracy vs. projection dimensionality.}
	\label{figure3}
\end{figure}

\begin{figure}[h!]
\begin{center}
		\mbox{
		\includegraphics[width=2.1in]{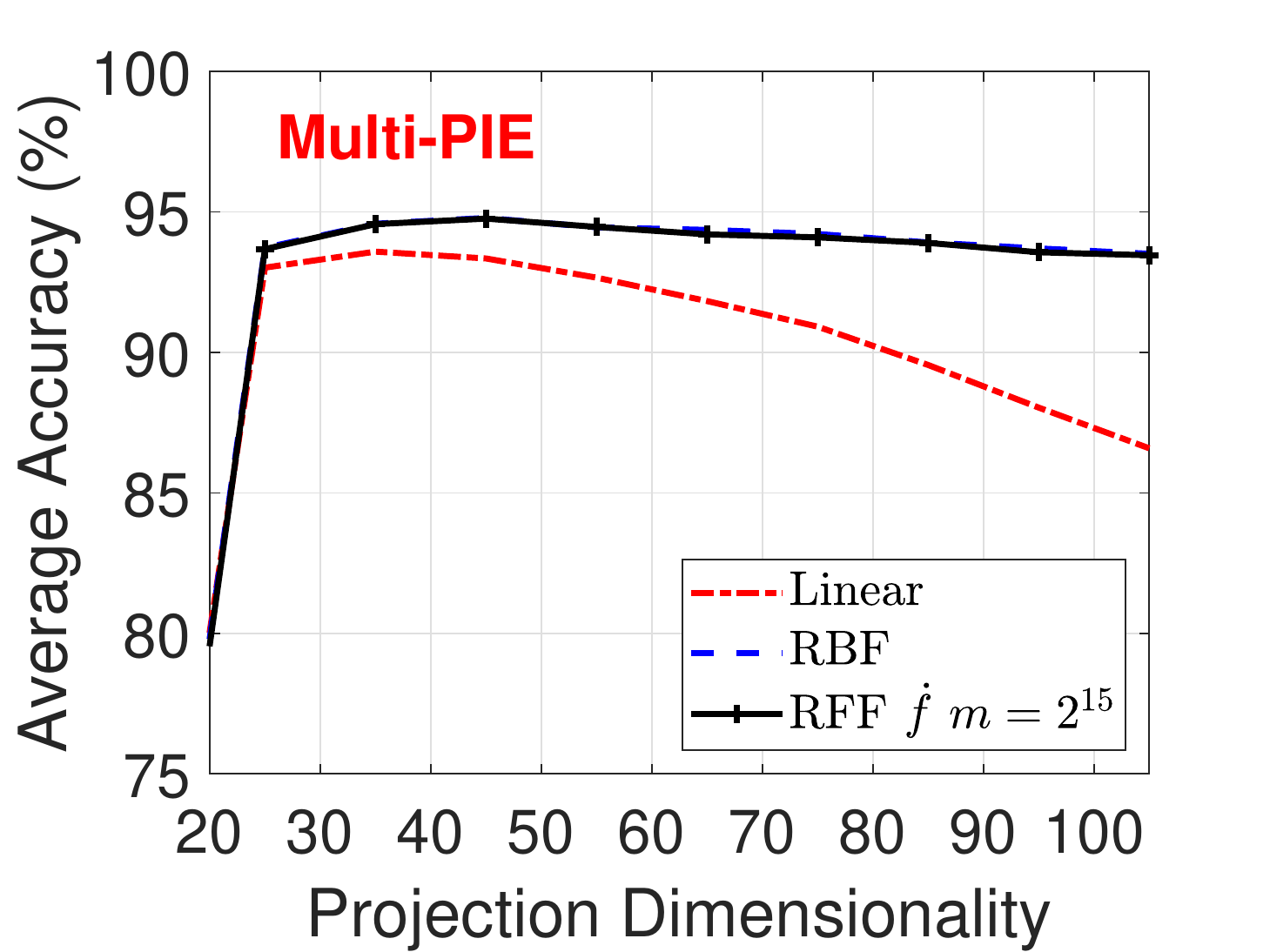}
		}
	\end{center}
 	\vspace{-0.2in}
	\caption{Multi-PIE dataset: rank-1 recognition rate vs. projection dimensionality of average recognition rate of all cross-views.}
\label{figure4}
\end{figure}


\section{Concluding Remarks}
In this present paper, we look into the problem of multi-view discriminant analysis, and incorporate kernel method to improve the learning performance. We seek to linearize the process by adopting random Fourier features to approximate the RBF kernel. Theoretical analysis on the change in eigenspace with such approximation is provided. We conduct experiments on various multi-view datasets to show that kernel MvDA notably improves vanilla MvDA in multi-view retrieval tasks, and using linearized kernels can well approximate the learning power of using the exact kernel in such problems. As multi-view model becomes more and more popular with many important applications in practice, we expect our work to be valuable for large-scale multi-view tasks, and motivate more research on randomized multi-view learning algorithms. Admittedly, this paper is just the beginning of the line of interesting work on randomized kernel multi-view learning and Authors look forward to seeing better (e.g., more accurate) algorithms and improved theory in the future.

\section*{Acknowledgement}
We thank the anonymous referees for their constructive comments.  The work was partially supported by NSF-III-1360971, NSF-Bigdata-1419210, ONRN00014-13-1-0764,  AFOSR-FA9550-13-231-0137, and NSFC-61572463. Jie Gui's work was conducted while he was a postdoctoral researcher at Rutgers University.\\

\bibliography{standard}

\end{document}